\documentclass{fairmeta}

\usepackage{amsmath,amssymb,mathtools,amsthm}
\usepackage{booktabs,array,colortbl,adjustbox}
\usepackage{float,enumitem,xspace,scalefnt}
\usepackage[nameinlink,capitalize]{cleveref}

\definecolor{linkcolor}{RGB}{83,83,182}
\definecolor{bestblue}{rgb}{0.70,0.82,0.95}
\definecolor{secondblue}{rgb}{0.88,0.94,1.00}
\newcommand{\best}[1]{\cellcolor{bestblue}\bfseries #1}
\newcommand{\second}[1]{\underline{#1}}

\newcommand{\Normal}{\mathcal{N}}
\newcommand{\Loss}{\mathcal{L}}
\newcommand{\dist}{\operatorname{dist}}
\newcommand{\diag}{\operatorname{diag}}
\newcommand{\algn}[1]{{\normalfont\scalefont{0.95}#1}\xspace}
\newcommand{\paretoq}{\algn{ParetoQ}}

\newtheorem{theorem}{Theorem}
\newtheorem{proposition}{Proposition}
\newtheorem{lemma}{Lemma}
\newtheorem{corollary}{Corollary}
\newtheorem{definition}{Definition}

\crefname{section}{Sec.}{Secs.}
\Crefname{section}{Sec.}{Secs.}
\crefname{table}{Tab.}{Tabs.}
\Crefname{table}{Tab.}{Tabs.}
\crefname{figure}{Fig.}{Figs.}
\Crefname{figure}{Fig.}{Figs.}
\crefname{equation}{Eq.}{Eqs.}
\Crefname{equation}{Eq.}{Eqs.}
\crefname{theorem}{Thm.}{Thms.}
\Crefname{theorem}{Thm.}{Thms.}
\crefname{lemma}{Lem.}{Lems.}
\Crefname{lemma}{Lem.}{Lems.}
\crefname{corollary}{Cor.}{Cors.}
\Crefname{corollary}{Cor.}{Cors.}
\crefname{proposition}{Prop.}{Props.}
\Crefname{proposition}{Prop.}{Props.}
\crefname{definition}{Def.}{Defs.}
\Crefname{definition}{Def.}{Defs.}
\crefname{appendix}{App.}{Apps.}
\Crefname{appendix}{App.}{Apps.}

\hypersetup{
  colorlinks=true,
  citecolor=linkcolor,
  linkcolor=linkcolor,
  pdftitle={Few Bits, One Law: Toward W2A4KV2},
  pdfauthor={Kai Yi, Tarek Elgamal, Sruthikesh Surineni, Vignesh Vivekraja, Soumyadeep Ghosh, Steven Li}
}
\title{Few Bits, One Law: Toward W2A4KV2}

\author[]{Kai Yi}
\author[]{Tarek Elgamal}
\author[]{Sruthikesh Surineni}
\author[]{Vignesh Vivekraja}
\author[]{Soumyadeep Ghosh}
\author[]{Steven Li}

\affiliation[]{Meta AI}

\abstract{
Extreme low-bit LLM compression is most challenging when weights, activations, and KV caches are quantized together: their distributions differ, and quantization errors interact throughout the network. We introduce \algn{CanonQ}, a unified quantization-aware training framework that addresses these challenges by separating source canonicalization from task-aware adaptation. Fixed rotations and energy normalization map heterogeneous tensor sources to canonical coordinates, enabling frozen Gaussian-reference codebooks to be reused across layers and models. Joint training then adapts the network to the coupled errors of weight, activation, and cache quantization within a common scalar/vector interface. We bound frozen-codebook transfer error and local task loss, and derive an exact normalization-aware straight-through Jacobian that links quantization distortion to gradient bias. The strongest gains arise under joint \texttt{W2A4KV2} compression: across \texttt{LLaMA3-1B/3B/8B}, \algn{CanonQ-Omni} achieves up to $14.28\times$ lower \texttt{WikiText-2} perplexity and up to $57.9\%$ higher mean zero-shot accuracy than prior state-of-the-art and representative quantization baselines. The benefits extend to \texttt{Qwen3-1.7B}, code generation, and mathematical reasoning: on instruction-tuned \texttt{MobileLLM-Pro-1B} at \texttt{W2A16KV16}, \algn{CanonQ} achieves relative improvements of $41.7\%$ in \texttt{HumanEval} pass@1 and $39.1\%$ in \texttt{GSM8K} exact match over the strongest evaluated quantization baseline.
}

\correspondence{\email{kaiyi96@meta.com}}

\begin{document}

\maketitle
\begin{figure}[!ht]
  \centering
  \includegraphics[width=\linewidth]{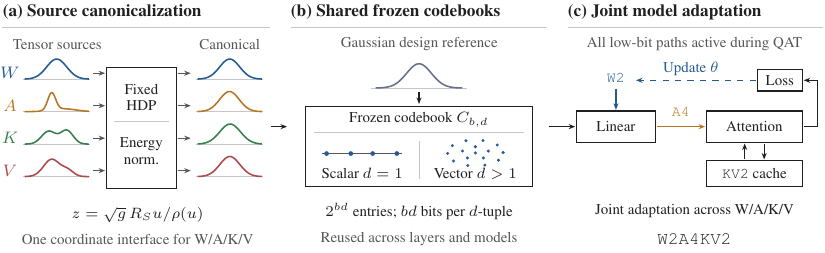}
  \caption{\textbf{Schematic overview of \algn{CanonQ}.}
  Fixed source-specific transforms and normalization map tensor blocks to canonical coordinates.
  Frozen Gaussian-reference codebooks are reused across layers and models for matching $(b,d)$.
  Joint QAT adapts the model to coupled quantization errors at \texttt{W2A4KV2}.
  }
  \label{fig:canonq_overview}
\end{figure}

\section{Introduction}
\label{sec:intro}
LLM inference is increasingly limited by memory movement: weights dominate model storage and matrix-multiplication bandwidth, activations affect prefill, and the KV cache grows with batch size and context length. Quantization is therefore central to deployment, yet the most attractive regimes, which include \texttt{W1/W2} weights, \texttt{A4} activations, and \texttt{KV2} caches, remain difficult. Post-Training Quantization (PTQ) is effective at moderate precision but often collapses when several components are compressed simultaneously \citep{dettmers2022llmint8,xiao2023smoothquant,frantar2022gptq,lin2024awq,shao2024omniquant,egiazarian2024aqlm,malinovskii2024pv}. Quantization-Aware Training (QAT) can recover quality by adapting the model to the quantized computation \citep{liu2024llmqat,liu2025paretoq,chen2025efficientqat,li2026winq,yi2026jacquant}, but the result still depends on the discrete representation to which the model is trained.

At such low precision, that representation is not a minor implementation choice. Uniform quantization uses equally spaced reconstruction levels, even though weights, activations, and KV states have different distributional shapes. A data-fitted codebook can accommodate those differences, but ties the representation to the tensors or calibration data used to construct it. Joint compression makes the choice harder: weights shape the activation distribution, and quantized activations generate the keys and values that must themselves survive a two-bit cache. This motivates a common quantization interface that supports adaptation across the coupled computation.

We ask: \emph{can a small quantization interface be designed once, then reused while an LLM learns to operate through it?} Our approach separates two decisions. Source-specific transforms and normalization determine the coordinates presented to the quantizer; a source-independent reference determines the reconstruction alphabet. QAT then adapts the network to the resulting fixed codec. We introduce \algn{CanonQ} around this principle: canonicalize the source, then train the model to use the codec. A fixed source-specific Hadamard--diagonal--permutation (HDP) transform mixes each block, and energy normalization removes its magnitude. The codebooks and HDP maps stay fixed throughout QAT. \algn{CanonQ-W} applies this interface to weights; \algn{CanonQ-Omni} applies it to the complete W/A/K/V computation.

Our central finding is that the benefit of canonical quantization becomes most pronounced when weights, activations, and KV caches are compressed together.

Our contributions are: \textbf{(1) A unified framework for joint W/A/KV quantization.} We introduce \algn{CanonQ}, which unifies scalar and vector quantization of weights, activations, and KV caches through a common canonical interface. Frozen reference codebooks are reusable across layers and models, while joint QAT adapts the network to interacting quantization errors across the full computation.
\textbf{(2) A unified theory linking source alignment to optimization.} We establish a quantitative chain from source mismatch to codebook distortion, local task loss, and optimization error. Wasserstein bounds characterize when frozen reference codebooks transfer to actual tensor distributions, while an exact normalization-aware straight-through estimator (STE) Jacobian identity reveals how assignment distortion controls backward sensitivity and gradient bias. We further derive distortion-dependent stationarity guarantees for a smooth reference objective, providing a theoretical foundation for the canonical training interface.
\textbf{(3) Large quality gains under extreme joint compression.} At \texttt{W2A4KV2}, \algn{CanonQ-Omni} achieves up to $14.28\times$ lower \texttt{WikiText-2} perplexity and up to $57.9\%$ higher mean zero-shot accuracy than strong quantization baselines across \texttt{LLaMA3-1B/3B/8B}. Further evaluations demonstrate transfer to \texttt{Qwen3-1.7B}, substantial coding and \texttt{GSM8K} improvements under \texttt{W2A16KV16}, and stable relative perplexity within the \texttt{8K} continued-QAT training range.

\section{Related Work}
\label{sec:related_work}
\textbf{Low-bit optimization and codebooks.}
PTQ methods including \algn{GPTQ}, \algn{AWQ}, and \algn{OmniQuant} improve low-bit reconstruction \citep{frantar2022gptq,lin2024awq,shao2024omniquant}; \algn{LLM-QAT}, \algn{EfficientQAT}, \algn{ParetoQ}, \algn{WinQ}, and \algn{JacQuant} adapt models through quantized training \citep{liu2024llmqat,chen2025efficientqat,liu2025paretoq,li2026winq,yi2026jacquant}.
Binary and ternary methods likewise make representation a first-order design choice \citep{wang2023bitnet,ma2024era,shang2023pbllm,huang2024billm,xu2024onebit}.
\algn{Lloyd--Max} and \algn{Linde--Buzo--Gray} provide classical scalar and vector constructions \citep{max1960quantizing,lloyd1982least,linde1980algorithm}; modern non-uniform, additive, and lattice approaches include \algn{NF4}, \algn{SqueezeLLM}, \algn{AQLM}, \algn{PV-Tuning}, and \algn{QuIP\#} \citep{dettmers2023qlora,kim2024squeezellm,egiazarian2024aqlm,malinovskii2024pv,tseng2024quipsharp}.
\algn{LC-QAT} is a close weight-only QAT comparison: it represents two-bit quantized weights through a learned affine map over discrete vectors and avoids explicit codebook lookup in the training forward pass \citep{wang2026lcqat}.
\algn{CanonQ} instead freezes the reference table, retains nearest-codeword assignment, and trains all enabled W/A/K/V paths.

\textbf{Rotations, canonicalization, and caches.}
Randomized rotations support distributed compression and LLM incoherence processing \citep{vargaftik2021drive,vargaftik2022eden,chee2023quip,tseng2024quipsharp}.
\algn{QuaRot} and \algn{SpinQuant} suppress outliers or learn rotations \citep{ashkboos2024quarot,liu2024spinquant}; \algn{KIVI} and \algn{KVQuant} exploit cache-specific structure \citep{liu2024kivi,hooper2024kvquant}.
\algn{TurboQuant} separates norm from direction and constructs scalar quantizers for rotation-induced marginals; its inner-product variant additionally quantizes a residual with \algn{QJL} \citep{zandieh2025turboquant}.
This is the closest source-modeling connection. We study task-aware adaptation to a fixed canonical codec, allow joint tuple assignment, and quantify the implemented HDP map's mismatch from its design reference. Building on joint four-bit compression exemplified by \algn{QuaRot}, \algn{CanonQ} targets the more aggressive \texttt{W1/W2-A4/A8-KV2} regime through a reusable canonical interface and joint task-aware adaptation.

\section{\algn{CanonQ}: Unified W/A/KV Quantization}
\label{sec:method}
\label{sec:preliminaries}
As illustrated in \cref{fig:canonq_overview}, \algn{CanonQ} separates source-specific coordinates from source-independent reconstruction choices. Let $u\in\mathbb R^{g_S}$ be a block from $S\in\{W,A,K,V\}$. Draw and fix the orthogonal map $R_S=P_SHD_S$, where $H$ is a normalized Hadamard matrix, $D_S$ contains random signs, and $P_S$ is a permutation. Each source shares its map across layers; keys and values share one KV map.
With $\rho_S(u)=\|u\|_2+\epsilon$, define
\begin{equation}
 z=\sqrt{g_S}R_Su/\rho_S(u).
 \label{eq:canonical_coordinates}
\end{equation}
Rotation redistributes energy; normalization removes magnitude. For $\epsilon=0$, the representation is invariant to positive rescaling of nonzero $u$. The Gaussian law $Q_d=\Normal(0,I_d)$ provides a reusable codebook design reference, motivated by the large-$g_S$ behavior of Haar-rotated sphere coordinates \citep{diaconis1984asymptotics}. The canonicalization defect in \cref{def:canonicalization_defect} quantifies deviations under the implemented structured HDP map.

\subsection{Reusable Codebooks in Canonical Coordinates}
For $b$ bits per scalar and tuple dimension $d$, the population target is the best $K=2^{bd}$-point quantizer of $Q_d$,
\begin{equation}
 C_{b,d}^{\star}\in\arg\min_{1\le |C|\le 2^{bd}}\frac1d\mathbb E_{V\sim Q_d}\dist(V,C)^2.
 \label{eq:canonical_vq_design}
\end{equation}
Generalized \algn{Lloyd/LBG} iterations approximate this target offline; we denote the deployed numerical table by $C^{\mathrm{can}}_{b,d}:=\widehat C_{b,d}$ and retain its finite fitting gap $\varepsilon_{\mathrm{vq}}$ in \cref{eq:canonical_fit_gap}.
For $d=1$ the construction is \algn{Lloyd--Max}; for $d>1$ assignment is genuinely joint rather than a Cartesian product of scalar decisions.
For example, \texttt{W2} with $d=4$ uses 256 four-dimensional codewords, so one 8-bit index represents four weights.
A table depends only on $(b,d)$ and stays frozen across groups, layers, models, and QAT updates.

Group energies and source maps handle scale and coordinates, while QAT adapts the model without refitting each tensor's alphabet. The shared construction therefore accommodates different raw W/A/K/V distributions.

Partition $z$ into consecutive tuples $v_j\in\mathbb R^{d_S}$, with $d_S\mid g_S$. The codec is
\begin{equation}
 \widehat v_j=\Pi_{C^{\mathrm{can}}_{b_S,d_S}}(v_j),\quad
 \widehat u=\gamma_{S,\ell}\frac{\rho_S(u)}{\sqrt{g_S}}R_S^\top\operatorname{concat}_j(\widehat v_j).
 \label{eq:canonical_codec}
\end{equation}
Let $q=\operatorname{concat}_j(\widehat v_j)$ denote the assigned canonical block. The nonnegative gain $\gamma_{S,\ell}$ adjusts its scale after assignment; clipping and decoding metadata remain source-specific. Here $g_S$ is the normalization-block size and $d_S$ is the assignment dimension. Assignment error $\|q-z\|_2$ and reconstruction error $\|\gamma_{S,\ell}q-z\|_2$ therefore describe different quantities. With identity assignment and unit gain, the codec recovers $u$ exactly.

\subsection{Joint Adaptation to Quantization Errors}
Nearest-codeword assignment uses the identity STE,
\[
 \widetilde v_j=v_j+\operatorname{sg}\!\left(\Pi_{C^{\mathrm{can}}_{b_S,d_S}}(v_j)-v_j\right),
 \quad
 \min_{\theta,\phi}\mathbb E_{(x,y)\sim\mathcal D}\ell\!\left(f_{Q_C(\theta;\phi)}(x),y\right),
\]
where $\operatorname{sg}$ stops gradients, $\theta$ denotes shadow-model parameters, and $\phi$ contains enabled trainable quantizer parameters.
The reported canonical activation and KV paths use energy normalization and frozen tables without an additional learned, absmax, or percentile clipping parameter.
All low-bit paths in \algn{CanonQ-Omni} are active during QAT, so training sees the coupled errors that the evaluation operator will introduce.
Evaluation re-quantizes the shadow weights and uses the labeled activation/cache path.
{The norm is recomputed from the current block. Differentiating it in both encoding and decoding makes the block STE differ from the identity, as derived in \cref{subsec:ste_analysis}.}

\subsection{Scalar/Vector Variants and Computational Cost}
\algn{CanonQ-W} canonicalizes weights and retains uniform quantizers for enabled A/KV paths; \algn{CanonQ-Omni} canonicalizes every enabled source. We also implement \algn{VQLLM-QAT}, a raw-space control with group-fitted weight codebooks, the same tuple dimension and index rate, nearest-codeword assignment, and STE. This control changes transformation and table fitting together.
Sources literally share a table only when $(b,d)$ matches; otherwise they share the construction.
The suffix \texttt{d} denotes $d_W$; reported canonical A/K/V paths use $d_A=d_K=d_V=1$.
Each tuple uses $bd$ index bits. A shared table occupies $2^{bd}d q_C$ bits for $q_C$-bit components. Exhaustive assignment costs $O(g_S2^{bd})$ arithmetic per block and the Hadamard transform costs $O(g_S\log g_S)$. These arithmetic costs motivate small $d$. Group metadata, table bytes, and unquantized modules are accounted for in \cref{app:bit_accounting}.

\section{Theory: From Source Alignment to Optimization}
\label{sec:theory}

The fixed interface raises three questions: when do reference codebooks transfer, how do quantization residuals affect local task loss, and how does distortion influence training?

\subsection{Bounding Frozen-Codebook Transfer Error}
\label{subsec:canonical_transfer}

Let $Q_d=\Normal(0,I_d)$. Haar-rotated normalized blocks converge to this reference in $W_2$ for fixed $d$ (\cref{prop:haar_reference}).
The implemented $R=PHD$ is a structured orthogonal transform. For blocks independent of its random signs, \cref{prop:hadamard_bridge} establishes sub-Gaussian coordinates, a marginal Gaussian-approximation bound controlled by $\|u\|_3^3/\|u\|_2^3$, and an exact mean tuple-covariance defect controlled by $g\|u\|_4^4/\|u\|_2^4$. Our transfer analysis explicitly accounts for deviations from the Gaussian reference through the checkpoint-specific joint tuple law.

Write $\omega=(\theta,\phi)$ and let $P_{S,d}^{\omega}$ denote the post-HDP tuple law induced by checkpoint $\omega$ under a fixed sampling rule over layers, groups, tokens, heads, and tuple positions.
\begin{definition}[Canonicalization defect]
\label{def:canonicalization_defect}
For any tuple law $P$ with finite second moment, define
\begin{equation}
    \Delta_d(P)
    :=
    \frac{1}{\sqrt d}W_2(P,Q_d),
    \quad
    \Delta_{S,d}(\omega)
    :=
    \Delta_d(P_{S,d}^{\omega}).
    \label{eq:canonicalization_defect}
\end{equation}
\end{definition}

This checkpoint-specific defect measures mismatch in the full joint tuple law, including coordinate dependence, structured-transform effects, and $\epsilon>0$ normalization. It connects the reference-codebook design to the tensor distributions encountered during QAT.

For a $K$-point codebook $C\subset\mathbb R^d$, let
\begin{equation}
    D_{P,d}(C)
    :=
    \frac1d\mathbb E_{V\sim P}\dist(V,C)^2,
    \quad
    \mathcal R_{P,d}(C):=\sqrt{D_{P,d}(C)},
    \label{eq:root_distortion}
\end{equation}
and let
$\mathcal R_{P,b,d}^{\star}:=\inf_{1\le|C|\le 2^{bd}}\mathcal R_{P,d}(C)$.
The numerical canonical table $\widehat C_{b,d}$ is allowed a squared-error
fitting gap
\begin{equation}
    D_{Q_d,d}(\widehat C_{b,d})
    \le
    (\mathcal R_{Q_d,b,d}^{\star})^2+\varepsilon_{\mathrm{vq}}.
    \label{eq:canonical_fit_gap}
\end{equation}

\begin{theorem}[Frozen-codebook transfer]
\label{thm:frozen_codebook_transfer}
Let $A_b$ be any $2^b$-level scalar alphabet and $A_b^d$ its Cartesian
product.  Then
\begin{equation}
    \mathcal R_{Q_d,b,d}^{\star}
    \le \mathcal R_{Q_d,d}(A_b^d),
    \quad
    \mathcal R_{Q_d,d}(\widehat C_{b,d})
    \le \mathcal R_{Q_d,b,d}^{\star}+\sqrt{\varepsilon_{\mathrm{vq}}}.
    \label{eq:same_rate_root_dominance}
\end{equation}
For every actual tuple law $P$ with finite second moment,
\begin{equation}
    \mathcal R_{P,d}(\widehat C_{b,d})
    -\mathcal R_{P,b,d}^{\star}
    \le
    2\Delta_d(P)+\sqrt{\varepsilon_{\mathrm{vq}}}.
    \label{eq:root_regret_transfer}
\end{equation}
More directly, define the reference margin
$m_Q(A_b^d,\widehat C_{b,d})
:=\mathcal R_{Q_d,d}(A_b^d)-
\mathcal R_{Q_d,d}(\widehat C_{b,d})$.  Then
\begin{equation}
    \mathcal R_{P,d}(\widehat C_{b,d})
    -\mathcal R_{P,d}(A_b^d)
    \le
    2\Delta_d(P)-m_Q(A_b^d,\widehat C_{b,d}).
    \label{eq:comparator_margin_transfer}
\end{equation}
Hence a validated lower bound on the reference shaping gain transfers whenever
it exceeds twice a validated upper bound on the canonicalization defect.
\end{theorem}

The bound separates two routes to reliable transfer: better reference-table fitting reduces $\varepsilon_{\mathrm{vq}}$, while better canonical coordinates reduce $\Delta_d(P)$. The margin condition in \cref{eq:comparator_margin_transfer} quantifies when a reference reconstruction advantage survives deployment mismatch. These assignment results use $\gamma=1$; \cref{lem:gain_correction} extends the accounting to assign-then-gain decoding. Tuple dimension affects reference distortion, fitting error, and joint source mismatch together. We next incorporate downstream sensitivity to connect these representation effects to local task loss.

\subsection{Bounding Local Task Loss from Quantization}
\label{subsec:loss_and_trainability}

Let $\Pi_C^{(d)}$ apply $C$ tuple by tuple to
$z=\sqrt g Ru/\rho(u)$.  The implementation first assigns against $C$ and
then applies a nonnegative source-specific gain $\gamma$.  Its actual canonical residual and
block perturbation are therefore
\begin{equation}
    \eta_{C,\gamma}(z):=\gamma\Pi_C^{(d)}(z)-z,
    \quad
    e_{C,\gamma}(u)
    :=\widehat u-u
    =\frac{\rho(u)}{\sqrt g}R^\top\eta_{C,\gamma}(z).
    \label{eq:actual_block_perturbation}
\end{equation}

Quantization errors also interact across sources. For a linear block, write $\widehat W=W+E_W$ and $\widehat a=a+e_A$. Then
\[
\widehat W\widehat a-Wa=E_Wa+We_A+E_We_A.
\]
{The cross term, and the analogous query/key/value terms in attention (\cref{lem:coupled_linear_attention}), motivate joint adaptation.}

\begin{theorem}[Gain-aware local task-loss certificate]
\label{thm:local_loss_certificate}
Consider a single-site intervention that replaces block $u$ by
$u+e_{C,\gamma}(u)$ while holding the other sites fixed.  Suppose $\Loss$ is
$C^3$ on a neighborhood of the joining segment,
$\|\nabla^2\Loss(u)\|_{\mathrm{op}}\le\Lambda_u$, and
$\|\nabla^3\Loss(w)\|_{\mathrm{op}}\le M_3$ on that segment.  Then
\begin{align}
    |\Loss(u+e_{C,\gamma})-\Loss(u)|
    &\le
    \frac{\rho(u)}{\sqrt g}
    \|\nabla\Loss(u)\|_2\|\eta_{C,\gamma}(z)\|_2
    \nonumber\\[-0.2em]
    &\quad+
    \frac{\rho(u)^2\Lambda_u}{2g}
    \|\eta_{C,\gamma}(z)\|_2^2
    +
    \frac{M_3\rho(u)^3}{6g^{3/2}}
    \|\eta_{C,\gamma}(z)\|_2^3.
    \label{eq:full_local_loss_bound}
\end{align}
\end{theorem}

The certificate applies to a smooth downstream continuation or a hard-quantized branch whose downstream assignments remain fixed throughout a neighborhood of the intervention segment. It retains the local gradient term explicitly, covering activation and KV interventions away from coordinate-wise stationarity, and separates first-order sensitivity, curvature, and higher-order effects.
For random blocks, set
$Z=\sqrt g RU/\rho(U)$ and choose
$\Lambda_U\ge\|\nabla^2\Loss(U)\|_{\mathrm{op}}$.  The quadratic contribution
is bounded by one half of the sensitivity-weighted distortion
\begin{equation}
    D^{\mathrm{curv}}_{S,\gamma}(C)
    :=\frac1g\mathbb E
    \left[\rho(U)^2\Lambda_U
    \|\eta_{C,\gamma}(Z)\|_2^2\right],
    \label{eq:curvature_weighted_distortion}
\end{equation}
without residual-isotropy assumptions. {Expected certificates are finite when the displayed sensitivity-weighted moments are finite. Common-cut interventions retain cross-site Hessian terms; causally ordered sites require telescoping with the smoothness condition checked at each step (\cref{cor:multisite_local_loss}).}

\subsection{Quantization Distortion Controls STE Error}
\label{subsec:ste_analysis}
{The hard nearest-codeword objective is generally discontinuous at cell boundaries. We therefore compare its STE with a smooth gain-preserving reference.} Set
$s(u)=\rho(u)/\sqrt g$, $q=\Pi_C^{(d)}(z)$, and
$\delta_a=\|q-z\|_2/\sqrt g$.  If the backward pass uses $Dq/Dz=I$ and treats
$\gamma$ as independent of $u$, for $u\ne0$, its surrogate Jacobian is
\begin{equation}
    J_{\mathrm{STE}}(u)
    =\gamma\left[I+R^\top(q-z)\nabla s(u)^\top\right],
    \quad
    \|J_{\mathrm{STE}}-\gamma I\|_{\mathrm{op}}
    =\gamma\delta_a.
    \label{eq:ste_jacobian_summary}
\end{equation}
Thus, an identity STE at assignment induces a distortion-dependent block correction of rank at most one. If the group energy is detached in both encoding and decoding, the local surrogate Jacobian becomes $\gamma I$ instead; additional clipping introduces its own derivative.

For a downstream loss $f$ with $L_f$-Lipschitz gradient, this identity yields
the explicit bias certificate
\begin{equation}
    \|g_{\mathrm{STE}}-\gamma\nabla f(\gamma u)\|_2
    \le
    \gamma\delta_a\|\nabla f(\gamma u)\|_2
    +\gamma^2(1+\delta_a)L_f\rho(u)\delta_a
    =:B_C(u).
    \label{eq:block_ste_bias_summary}
\end{equation}
For the uniform tuple law $P_u$ of a fixed group, $\delta_a(u)^2=D_{P_u,d}(C)$. Under matching random-group/tuple sampling, $\mathbb E\delta_a(U)^2=D_{P,d}(C)$. If gains, energies, downstream gradients, and assignment residuals are bounded on the analyzed region, then $B_C\le K\delta_a$ and $\mathbb E B_C(U)^2\le K^2D_{P,d}(C)$ (\cref{cor:rms_ste_bias}). A trained gain contributes an additional Jacobian-column defect $\rho(u)\delta_a$. These give codec-dependent sufficient conditions, rather than assuming alignment independently of the quantizer \citep{yin2019understanding}.

For the full network, let $F$ be the $L$-smooth gain-preserving reference obtained by replacing {all enabled hard quantizers (including any uniform control paths)} by the identity. Let $G_t$ be the stochastic STE direction, $\mathbb E_t[\cdot]=\mathbb E[\cdot\mid\mathcal F_t]$ condition on the pre-minibatch history, and $b_t=\mathbb E_tG_t-\nabla F(\omega_t)$. Under backward-map stability, integrability, and valid differentiation under expectation, \cref{lem:ste_bias_from_codec} lifts the codec-dependent errors to a full-network bias certificate $\|b_t\|_2\le B_t$.

\begin{theorem}[Distortion-controlled STE stationarity]
\label{thm:distortion_controlled_ste}
Assume $F_\star> -\infty$, $\|b_t\|_2\le B_t$, and the conditional STE
variance satisfies
$\mathbb E_t\|G_t-\mathbb E_tG_t\|_2^2\le\sigma^2$.  For
$\omega_{t+1}=\omega_t-\alpha_tG_t$ with deterministic steps
$0<\alpha_t\le 1/(4L)$ and
$S_T=\sum_{t<T}\alpha_t$,
\begin{equation}
    \frac{1}{S_T}\sum_{t<T}\alpha_t
    \mathbb E\|\nabla F(\omega_t)\|_2^2
    \le
    \frac{2(F(\omega_0)-F_\star)}{S_T}
    +\frac{5}{2S_T}\sum_{t<T}\alpha_t\mathbb E B_t^2
    +\frac{L\sigma^2}{S_T}\sum_{t<T}\alpha_t^2.
    \label{eq:distortion_controlled_stationarity}
\end{equation}
\end{theorem}

With $\alpha_t=c/\sqrt T\le1/(4L)$, the bound separates $O(T^{-1/2})$ optimization and sampling terms from the codec-dependent gradient-bias contribution. This establishes distortion-dependent stationarity bounds for the smooth gain-preserving reference under the stated SGD update rule. The appendix further develops a conditional PL bound and Voronoi-margin stability.

\textbf{Connection to evidence.} The analysis separates source mismatch, sensitivity-weighted reconstruction error, and assignment distortion. Scalar $W_1$/KS diagnostics assess marginal alignment; zero-update/QAT comparisons test adaptation to the fixed codec; precision and dimension sweeps establish how these choices affect model quality.

\section{Experiments}
\label{sec:experiments}

The evaluation first locates the benefit across compression settings, then examines source alignment and adaptation, and finally tests longer-context likelihood, code generation, and mathematical reasoning.

\subsection{Evaluation Setup}
\label{subsec:setup}
We evaluate \texttt{LLaMA3-1B/3B/8B} and \texttt{Qwen3-1.7B} \citep{llama3,yang2025qwen3}, using \texttt{1B} for the broad component study. Standard QAT uses \texttt{FineWeb-Edu}, sequence length \texttt{2048}, effective batch size \texttt{64}, \algn{AdamW}, and up to \texttt{240K} updates ($\approx31.5$B tokens). Learning rates are selected from $\{1,2,4\}\times10^{-5}$ on held-out validation data. The \texttt{8B} \algn{CanonQ-Omni} configuration reuses the selected \texttt{3B} setting. The context study separately continues QAT at \texttt{8192} for \texttt{40K} steps.
We report \texttt{WikiText-2} perplexity
(PPL) \citep{merity2017pointer}
and the unweighted mean zero-shot accuracy over eight benchmarks. Accuracy is reported in percent;
task definitions and evaluation details appear in
\cref{app:evaluation_protocol}.

{The instruction-tuned extension reports greedy coding pass@1 and \texttt{GSM8K} final-answer exact match under a separate protocol (\cref{app:coding_protocol}) \citep{austin2021program,chen2021evaluating,cobbe2021training}.}

\algn{ParetoQ} denotes its weight quantizer combined with the uniform A/KV paths specified in \cref{app:quantizer_formulations}. Uniform \texttt{A4/KV2} uses 15/3 active levels; canonical scalar tables use 16/4 at the same index rate.

Comparisons evaluate complete trained configurations, including these differences. Full settings and per-task results appear in \cref{app:experimental_details,subsec:precision_scaling}. Each reported configuration uses one training seed.

\begin{table}[t]
\caption{Joint-compression and cross-model results. Columns report PPL $\downarrow$ and eight-task accuracy $\uparrow$. The suffix denotes $d_W$; canonical A/K/V paths use $d=1$. \algn{CanonQ-Omni} uses canonical KV quantization in the headline \texttt{KV2} results. Full task scores appear in \cref{tab:complete_w2a8kv2_w2a4kv2,tab:complete_3b,tab:crossmodel_tasks}.}
\label{tab:main_summary_four_models}
\scriptsize
\setlength{\tabcolsep}{2.6pt}
\renewcommand{\arraystretch}{1.12}

\begin{minipage}[t]{0.485\textwidth}
\vspace{0pt}
\centering

\textbf{\texttt{LLaMA3-1B}, quantized KV}
\par\vspace{3pt}
\renewcommand{\arraystretch}{1.3}
\begin{tabular*}{\linewidth}{
  @{\extracolsep{\fill}}lrrrr@{}
}
\toprule
& \multicolumn{2}{c}{\texttt{W2A8KV2}} & \multicolumn{2}{c}{\texttt{W2A4KV2}} \\
\cmidrule(lr){2-3} \cmidrule(l){4-5}
Method & \multicolumn{1}{c}{PPL $\downarrow$} & \multicolumn{1}{c}{Acc. $\uparrow$} & \multicolumn{1}{c}{PPL $\downarrow$} & \multicolumn{1}{c}{Acc. $\uparrow$} \\
\midrule

FP Model & 9.75 & 58.54 & 9.75 & 58.54 \\
\midrule

\algn{RTN} (PTQ) & 5.8e5 & 37.31 & 5.5e5 & 36.72 \\
\algn{GPTQ} (PTQ) & 5.3e4 & 38.23 & 9.5e4 & 37.83 \\
\algn{QuaRot} (PTQ) & 3.4e4 & 37.39 & 4.1e4 & 38.11 \\
\midrule
\algn{KIVI} & 1.1e5 & 34.85 & 3.2e5 & 35.16 \\
\midrule

\paretoq & 47.58 & 45.13 & 90.57 & 40.47 \\
\midrule

\algn{VQLLM-QAT-d4} & 34.35 & 45.35 & 76.51 & 42.32 \\
\midrule

\algn{CanonQ-W-d1} & 30.07 & 46.41 & 64.40 & 42.49 \\

\algn{CanonQ-W-d2} & 30.19 & 47.56 & 64.48 & 42.95 \\

\algn{CanonQ-W-d4} & 31.27 & 46.91 & 76.19 & 41.97 \\

\midrule

\algn{CanonQ-Omni-d1}
& \best{15.70} & \best{53.45}
& \best{15.42} & \best{53.56} \\

\algn{CanonQ-Omni-d2}
& \second{16.43} & \second{52.67}
& \second{15.90} & \second{53.16} \\

\bottomrule
\end{tabular*}

\end{minipage}
\hfill
\begin{minipage}[t]{0.485\textwidth}
\vspace{0pt}
\centering

\textbf{\texttt{LLaMA3-3B}}
\par\vspace{3.5pt}
\renewcommand{\arraystretch}{0.95}
\begin{tabular*}{\linewidth}{
  @{\extracolsep{\fill}}lrrrr@{}
}
\toprule
& \multicolumn{2}{c}{\texttt{W2A4}}
& \multicolumn{2}{c}{\texttt{W2A4KV2}} \\
\cmidrule(lr){2-3}
\cmidrule(l){4-5}
Method
& \multicolumn{1}{c}{PPL $\downarrow$}
& \multicolumn{1}{c}{Acc. $\uparrow$}
& \multicolumn{1}{c}{PPL $\downarrow$}
& \multicolumn{1}{c}{Acc. $\uparrow$} \\
\midrule

FP Model
& 7.73 & 65.11
& 7.73 & 65.11 \\

\midrule

\algn{RTN} (PTQ)
& 4.2e5 & 35.74
& 1.1e6 & 36.78 \\

\algn{GPTQ} (PTQ)
& 7.9e4 & 38.12
& 8.9e4 & 37.88 \\

\algn{QuaRot} (PTQ)
& 6.7e2 & 37.23
& 1.5e4 & 37.95 \\

\midrule

\paretoq
& 17.88 & 53.15
& 92.64 & 39.39 \\

\midrule

\algn{VQLLM-QAT-d4}
& 30.71 & 46.30
& 52.82 & 42.07 \\

\midrule

\algn{CanonQ-Omni-d1}
& \second{10.66} & \best{61.16}
& \best{10.93} & \second{58.37} \\

\algn{CanonQ-Omni-d2}
& \best{10.50} & \second{61.14}
& \second{11.16} & \best{60.04} \\

\bottomrule
\end{tabular*}

\vspace{6pt}

\textbf{\texttt{Qwen3-1.7B} and \texttt{LLaMA3-8B}}
\par\vspace{3pt}
\renewcommand{\arraystretch}{0.95}
\begin{tabular*}{\linewidth}{
  @{\extracolsep{\fill}}lrrrr@{}
}
\toprule
& \multicolumn{2}{c}{\texttt{Q3-1.7B (W2A4)}}
& \multicolumn{2}{c}{\texttt{L3-8B (W2A4KV2)}} \\
\cmidrule(lr){2-3}
\cmidrule(l){4-5}
Method
& \multicolumn{1}{r}{PPL $\downarrow$}
& \multicolumn{1}{r}{Acc. $\uparrow$}
& \multicolumn{1}{r}{PPL $\downarrow$}
& \multicolumn{1}{r}{Acc. $\uparrow$} \\
\midrule

FP Model
& 16.33 & 58.10
& 6.14 & 70.14 \\

\midrule

\paretoq
& 43.56 & 42.35
& 134.36 & 37.87 \\

\midrule

\algn{CanonQ-Omni-d1}
& \best{17.80} & \best{50.96}
& \best{9.41} & \best{59.81}\\

\bottomrule
\end{tabular*}

\end{minipage}
\end{table}

\subsection{Large Quality Gains at W2A4KV2}
\label{subsec:main_results}
The central result is the large separation under \texttt{W2A4KV2} (\cref{tab:main_summary_four_models}). On \texttt{LLaMA3-1B}, \algn{CanonQ-Omni-d1} reaches 15.42 PPL and 53.56\% mean accuracy, versus 90.57 and 40.47\% for \algn{ParetoQ}: $5.87\times$ lower PPL and a 13.09-point accuracy gain. Relative to weight-only canonicalization with \algn{CanonQ-W-d1}, it lowers PPL by $4.18\times$ and raises accuracy by 11.07 points. This direct comparison locates the larger benefit in extending the canonical interface to activations and caches.

The gains persist at larger scales. On \texttt{3B}, \algn{CanonQ-Omni-d2} reaches 11.16/60.04 (PPL/accuracy), versus 92.64/39.39 for \algn{ParetoQ}: $8.30\times$ lower PPL and a 20.65-point gain. At \texttt{8B}, \algn{CanonQ-Omni-d1} reaches 9.41/59.81 versus 134.36/37.87, giving $14.28\times$ lower PPL and $57.9\%$ higher accuracy. To reduce tuning costs, we reuse the best training hyperparameters identified on \texttt{3B} for \texttt{8B}. We expect a dedicated \texttt{8B} hyperparameter search to further improve accuracy and narrow the gap to FP.

On \texttt{Qwen3-1.7B W2A4}, \algn{CanonQ-Omni-d1} obtains 17.80/50.96 versus 43.56/42.35, improving PPL by $2.45\times$ and accuracy by 8.61 points. Together, these results show that the same frozen-codebook
construction supports joint quantization across model scales
and transfers to a second model family.

\subsection{Gains Grow Beyond Weight-Only Quantization}
\label{subsec:precision_scaling_main}
\textbf{Weight-only quantization gives modest differences.} At \texttt{1B W2A16}, \algn{ParetoQ} and \algn{CanonQ-W-d2} achieve similar quality: 12.66/55.53 and 12.86/55.45 (PPL/accuracy), respectively (\cref{tab:llama3_1b_summary_no_kv}). At \texttt{W1A16}, \algn{CanonQ-W-d1} lowers PPL from 18.13 to 16.93, while $d_W=2$ gives the highest accuracy, 51.81\%. At matched $d_W=4$, canonical coding improves raw-space \algn{VQLLM-QAT} from 22.36/47.90 to 17.47/51.17, showing a clearer benefit from canonical coordinates and reference codebooks at one-bit weight precision.

\textbf{Activation compression produces the larger separation.} At \texttt{W2A8}, \algn{CanonQ-Omni-d1} improves \algn{ParetoQ}'s 18.22/51.09 to 13.25/55.60. At \texttt{W2A4}, weight-only canonicalization remains near 35 PPL and 44\% accuracy, whereas \algn{CanonQ-Omni-d1} reaches 13.84/54.71: $2.35\times$ lower PPL and a 10.59-point gain over \algn{ParetoQ}. All eight tasks improve, with gains from 1.51 points on \texttt{BoolQ} to 23.62 on \texttt{HellaSwag} (\cref{tab:complete_w2a8_w1a8_w2a4,fig:quality_decomposition}). Moving from \texttt{A8} to \texttt{A4} raises canonical-path PPL by only 0.59 and lowers accuracy by 0.89 points, versus 14.35 PPL and 6.97 points for \algn{ParetoQ}. Extending the interface to activations thus preserves quality as activation precision tightens.

\begin{table*}[t]
\caption{\texttt{LLaMA3-1B} results without KV-cache quantization. Best and second-best non-FP results in each column are shaded and underlined, respectively. Full task scores appear in \cref{tab:complete_w2a16_w1a16,tab:complete_w2a8_w1a8_w2a4}.}
\label{tab:llama3_1b_summary_no_kv}

\scriptsize
\setlength{\tabcolsep}{3.0pt}
\renewcommand{\arraystretch}{1.0}

\begin{minipage}[t]{0.395\textwidth}
\vspace{0pt}
\centering

\textbf{(a) Weight-only quantization}
\par\vspace{2.0pt}
\renewcommand{\arraystretch}{1.2}
\begin{adjustbox}{max width=\linewidth}
\begin{tabular}{@{}lrrrr@{}}
\toprule
& \multicolumn{2}{c}{\texttt{W1A16}} & \multicolumn{2}{c}{\texttt{W2A16}} \\
\cmidrule(lr){2-3} \cmidrule(l){4-5}
Method & \multicolumn{1}{c}{PPL $\downarrow$} & \multicolumn{1}{c}{Acc. $\uparrow$} & \multicolumn{1}{c}{PPL $\downarrow$} & \multicolumn{1}{c}{Acc. $\uparrow$} \\
\midrule

FP Model & 9.75 & 58.54 & 9.75 & 58.54 \\
\midrule

\algn{RTN} (PTQ)
& 4.2e8 & 33.7
& 1.2e6 & 36.40 \\

\algn{GPTQ} (PTQ)
& 7.9e4 & 38.03
& 1.2e4 & 38.21 \\

\algn{QuaRot} (PTQ)
& 1.3e4 & 38.34
& 3.2e3 & 38.25 \\

\midrule

\paretoq
& 18.13 & \second{51.55}
& \best{12.66} & \best{55.53} \\

\midrule

\algn{VQLLM-QAT-d4}
& 22.36 & 47.90
& 13.77 & 55.12 \\

\algn{VQLLM-QAT-d8}
& 20.10 & 49.85
& N/A & N/A \\

\midrule

\algn{CanonQ-W-d1}
& \best{16.93} & 51.42
& 12.94 & 55.11 \\

\algn{CanonQ-W-d2}
& \second{17.41} & \best{51.81}
& \second{12.86} & \second{55.45} \\

\algn{CanonQ-W-d4}
& 17.47 & 51.17
& 12.96 & \second{55.45} \\

\algn{CanonQ-W-d8}
& 18.79 & 50.56
& N/A & N/A \\

\bottomrule
\end{tabular}
\end{adjustbox}

\end{minipage}
\hfill
\begin{minipage}[t]{0.585\textwidth}
\vspace{0pt}
\centering

\textbf{(b) Weight--activation quantization}
\par\vspace{2pt}
\renewcommand{\arraystretch}{1.0}
\begin{adjustbox}{max width=\linewidth}
\begin{tabular}{@{}lrrrrrr@{}}
\toprule
& \multicolumn{2}{c}{\texttt{W1A8}}
& \multicolumn{2}{c}{\texttt{W2A8}}
& \multicolumn{2}{c}{\texttt{W2A4}} \\
\cmidrule(lr){2-3}
\cmidrule(lr){4-5}
\cmidrule(l){6-7}
Method
& \multicolumn{1}{c}{PPL $\downarrow$}
& \multicolumn{1}{c}{Acc. $\uparrow$}
& \multicolumn{1}{c}{PPL $\downarrow$}
& \multicolumn{1}{c}{Acc. $\uparrow$}
& \multicolumn{1}{c}{PPL $\downarrow$}
& \multicolumn{1}{c}{Acc. $\uparrow$} \\
\midrule

FP Model
& 9.75 & 58.54
& 9.75 & 58.54
& 9.75 & 58.54 \\

\midrule

\algn{RTN} (PTQ)
& 1.1e6 & 34.94
& 1.2e6 & 36.14
& 1.1e6 & 36.78 \\

\algn{GPTQ} (PTQ)
& 1.0e5 & 38.09
& 1.3e4 & 38.16
& 6.9e4 & 36.86 \\

\algn{QuaRot} (PTQ)
& 1.3e4 & 38.36
& 3.2e3 & 38.42
& 4.6e3 & 38.02 \\

\midrule

\paretoq
& 21.27 & 48.65
& 18.22 & 51.09
& 32.57 & 44.12 \\

\midrule

\algn{VQLLM-QAT-d4}
& 25.98 & 46.74
& 16.99 & 52.73
& 42.65 & 43.45 \\

\algn{VQLLM-QAT-d8}
& 25.37 & 47.47
& N/A & N/A
& N/A & N/A \\

\midrule

\algn{CanonQ-W-d1}
& 23.54 & 48.37
& 14.38 & \second{54.76}
& 35.16 & 44.14 \\

\algn{CanonQ-W-d2}
& 21.23 & 49.45
& 13.49 & 53.82
& 35.20 & 43.66 \\

\algn{CanonQ-W-d4}
& 22.10 & 48.75
& 16.01 & 52.95
& 35.53 & 44.53 \\

\algn{CanonQ-W-d8}
& 20.54 & 49.59
& N/A & N/A
& N/A & N/A \\

\midrule

\algn{CanonQ-Omni-d1}
& \second{17.76} & \best{51.25}
& \best{13.25} & \best{55.60}
& \best{13.84} & \best{54.71} \\

\algn{CanonQ-Omni-d2}
& \best{17.47} & \second{49.96}
& \second{13.42} & 54.63
& \second{14.51} & \second{54.53} \\

\bottomrule
\end{tabular}
\end{adjustbox}

\end{minipage}

\end{table*}

\textbf{Vector dimension changes the tradeoff.} At \texttt{1B W2A4KV2}, scalar \algn{CanonQ-Omni-d1} leads $d_W=2$ on both metrics. At \texttt{3B}, $d_W=2$ raises accuracy by 1.67 points while PPL increases from 10.93 to 11.16. The reference-law comparison in \cref{thm:frozen_codebook_transfer} explains the reconstruction choices available at a fixed index rate; task-level results determine which choice is useful after adaptation. Together, these results support selecting dimension by validation quality and lookup cost within the shared interface.

\subsection{Canonicalization and QAT Play Complementary Roles}
We next examine the complementary roles of source canonicalization and model adaptation through distribution diagnostics, zero-update comparisons, and cache-path ablations.

\textbf{HDP aligns dynamic-source marginals.}
Layer-mean $W_1$ distances to the Gaussian reference fall from 0.101 to 0.010 for activations, 0.060 to 0.026 for keys, and 0.030 to 0.018 for values; weights change from 0.015 to 0.017 (\cref{fig:hdp_collapse}). The largest distributional change is in activations, which also show the strongest separation between weight-only and full-stack canonicalization at \texttt{A4}. Keys and values move closer to the same reference, while weights need little marginal correction. These marginal measurements support the use of a shared reference construction for heterogeneous W/A/K/V sources.

\textbf{The fixed codec benefits from task-aware adaptation.}
At \texttt{1B W2A4}, applying the scalar canonical W/A codec with zero updates gives 7043.43/38.83, versus 13.84/54.71 after QAT (\cref{tab:ptq_vs_qat}). {PPL falls approximately $509\times$ while the table stays fixed. The codec supplies reconstruction choices, but the pretrained network must still adapt to using them throughout the computation. Marginal alignment and task-aware training thus address different parts of the problem. The \texttt{KV2} zero-update comparison also changes the cache path and measures their combined effect.}

\textbf{The cache path also matters.}
The uniform-KV \algn{CanonQ-Omni-d1} ablation reaches 22.82/47.72 at \texttt{W2A8KV2} and 21.65/48.30 at \texttt{W2A4KV2}, compared with 15.70/53.45 and 15.42/53.56 for canonical KV. The corresponding accuracy gains are 5.73 and 5.26 points (\cref{tab:cache_path_ablation}), so the cache contribution persists at both activation precisions. At \texttt{1B}, moving from \texttt{W2A4} to \texttt{W2A4KV2} costs the complete canonical path 1.58 PPL and 1.15 accuracy points, compared with 58.00 PPL and 3.65 points for \algn{ParetoQ}. The ablation evaluates the complete canonical cache path, combining transformation, active levels, normalization, and backward behavior.

\subsection{Stable Relative Perplexity Within the 8K Training Range}
After \texttt{8K} continued QAT for \texttt{40K} updates, we evaluate \texttt{LLaMA3-1B W2A4KV2} at \texttt{2K/4K/8K}. At \texttt{8K}, \algn{CanonQ-Omni} achieves 14.86 PPL, versus 39.50 for \algn{CanonQ-W} and 45.63 for \algn{ParetoQ}. Its PPL relative to the same-length FP reference stays within $1.69$--$1.74\times$, showing a stable relative quality gap as the evaluation window grows within the trained range. Full results and protocol appear in \cref{app:long_context_8k}.

\begin{table}[t]
\caption{{Code generation and mathematical reasoning on \texttt{MobileLLM-Pro-Instruct-1B W2A16KV16}. Scores are percentages, with correct-task counts in parentheses. \texttt{MBPP/HumanEval} report greedy pass@1; \texttt{GSM8K} reports final-answer exact match (EM). Quantized settings share the teacher and evaluation step. \texttt{GSM8K} also guides learning-rate selection; see \cref{app:coding_protocol}.}}
\label{tab:coding}
\centering
\small
\setlength{\tabcolsep}{4pt}
\begin{tabular}{@{}llrrrrr@{}}
\toprule
Benchmark & {Metric} & Tasks & FP base & \algn{ParetoQ} & \algn{VQLLM-QAT-d4} & \algn{CanonQ-d1}\\\midrule
\texttt{MBPP} & {pass@1} & 500 & 46.4 (232) & 28.0 (140) & 27.8 (139) & \textbf{32.2} (161)\\
\texttt{HumanEval} & {pass@1} & 164 & 60.4\ \ \ (99) & 43.9\ \ \  (72) & 33.5\ \ \  (55) & \textbf{62.2} (102)\\
\texttt{GSM8K} & {EM} & 1319 & 53.8 (709) & 34.3 (453) & 32.7 (431) & \textbf{47.8} (630)\\
\bottomrule
\end{tabular}
\end{table}

\subsection{Gains Extend to Code Generation and Reasoning}
\label{sec:coding}

We evaluate \texttt{MobileLLM-Pro-Instruct-1B} at \texttt{W2A16KV16} using the shared-teacher protocol in \cref{app:coding_protocol}. These experiments extend the evaluation to generated solutions under weight quantization.

\textbf{Mathematical reasoning.} On \texttt{GSM8K}, \algn{CanonQ-d1} solves 630 of 1,319 problems (47.8\%), versus 453 for \algn{ParetoQ} and 431 for \algn{VQLLM-QAT-d4}. This gives a $39.1\%$ relative improvement over the stronger quantized baseline and retains $88.9\%$ of the FP score. \texttt{GSM8K} guides learning-rate selection as specified in \cref{app:coding_protocol}.

\textbf{Code generation.} \algn{CanonQ-d1} achieves 32.2\% \texttt{MBPP} and 62.2\% \texttt{HumanEval} greedy pass@1 (\cref{tab:coding}), exceeding \algn{ParetoQ} by 4.2 and 18.3 points. The \texttt{HumanEval} gain is a $41.7\%$ relative improvement. The FP reference scores 60.4\% (paired comparison: $p=0.74$). Neither coding benchmark guides selection, and \algn{CanonQ} leads the evaluated quantized baselines on all three generation benchmarks.

\FloatBarrier
\section{Discussion and Limitations}
\label{sec:discussion}
\vspace{-2mm}

\textbf{A shared interface for coupled errors.} CanonQ separates reusable reconstruction choices from task-aware adaptation. Canonicalization makes a fixed reference alphabet useful across heterogeneous sources, while joint QAT adapts the network to their interacting errors. The larger gains under joint compression, together with the alignment and zero-update studies, support this division of roles. These results favor selecting the codec under the complete target precision configuration: methods with similar weight-only quality separate sharply when activations and caches are also quantized. Evaluating the intended W/A/KV format therefore captures interactions that weight-only comparisons leave untested.

\textbf{Choosing scalar or vector assignment.} The interface supports both choices at a fixed index rate. Scalar tables give the strongest reported \texttt{1B W2A4KV2} result; vector assignment improves \texttt{3B} accuracy with a small PPL tradeoff. Dimension is therefore a practical choice guided by validation quality and lookup cost. Joint tuple statistics and sensitivity-weighted residuals offer concrete quantities for explaining these differences.

\paragraph{Deployment and further evaluation.}
The present study establishes model quality at extreme low bit rates. Deployment evaluation will complement these results with exported-model size, KV-storage, and prefill/decode throughput measurements for fused kernels, accounting for group metadata and unquantized modules (\cref{app:bit_accounting}). Evaluation beyond the \texttt{8K} training range, reasoning benchmarks with independent hyperparameter selection, and repeated training seeds are directions for extending this evidence.

\section{Conclusion}
\label{sec:conclusion}
\vspace{-2mm}

We introduced \algn{CanonQ}, a unified framework for weight, activation, and KV-cache quantization using canonical coordinates and frozen reference codebooks. Our theory links source mismatch to reconstruction, task loss, and normalization-aware STE bias. Experiments establish substantial quality gains under \texttt{W2A4KV2} across model scales, with further improvements in coding and reasoning and stable relative perplexity within the \texttt{8K} training range.


\bibliographystyle{meta/assets/canonq}
\bibliography{references}

@article{benbasat2026rht,
  title         = {Quantizing With Randomized Hadamard Transforms: Efficient Heuristic Now Proven},
  author        = {Ben-Basat, Ran and Kuszmaul, William and Mitzenmacher, Michael and Portnoy, Amit and Vargaftik, Shay},
  journal       = {arXiv preprint arXiv:2605.06014},
  year          = {2026},
  eprint        = {2605.06014},
  archivePrefix = {arXiv},
  url           = {https://arxiv.org/abs/2605.06014}
}

@inproceedings{yin2019understanding,
  title     = {Understanding Straight-Through Estimator in Training Activation Quantized Neural Nets},
  author    = {Yin, Penghang and Lyu, Jiancheng and Zhang, Shuai and Osher, Stanley and Qi, Yingyong and Xin, Jack},
  booktitle = {International Conference on Learning Representations},
  year      = {2019},
  url       = {https://arxiv.org/abs/1903.05662}
}

@article{li2026winq,
  title={WinQ: Accelerating Quantization-Aware Training of Language Models Around Saddle Points},
  author={Li, Dongyue and Liu, Zechun and Yi, Kai and Zhang, Zhenshuo and Zhao, Changsheng and Krishnamoorthi, Raghuraman and Khaitan, Harshit and Zhang, Hongyang R and Li, Steven},
  journal={ICML},
  year={2026}
}

@article{linde1980algorithm,
  title   = {An Algorithm for Vector Quantizer Design},
  author  = {Linde, Yoseph and Buzo, Andres and Gray, Robert M.},
  journal = {IEEE Transactions on Communications},
  volume  = {28},
  number  = {1},
  pages   = {84--95},
  year    = {1980},
  month   = jan,
  doi     = {10.1109/TCOM.1980.1094577}
}

@inproceedings{vargaftik2021drive,
  title     = {{DRIVE}: One-bit Distributed Mean Estimation},
  author    = {Vargaftik, Shay and Ben-Basat, Ran and Portnoy, Amit and
               Mendelson, Gal and Ben-Itzhak, Yaniv and Mitzenmacher, Michael},
  booktitle = {Advances in Neural Information Processing Systems},
  volume    = {34},
  pages     = {362--377},
  year      = {2021},
  publisher = {Curran Associates, Inc.},
  url       = {https://proceedings.neurips.cc/paper/2021/hash/0397758f8990c1b41b81b43ac389ab9f-Abstract.html}
}

@inproceedings{vargaftik2022eden,
  title     = {{EDEN}: Communication-Efficient and Robust Distributed Mean
               Estimation for Federated Learning},
  author    = {Vargaftik, Shay and Ben Basat, Ran and Portnoy, Amit and
               Mendelson, Gal and Ben Itzhak, Yaniv and Mitzenmacher, Michael},
  booktitle = {Proceedings of the 39th International Conference on Machine Learning},
  editor    = {Chaudhuri, Kamalika and Jegelka, Stefanie and Song, Le and
               Szepesvari, Csaba and Niu, Gang and Sabato, Sivan},
  series    = {Proceedings of Machine Learning Research},
  volume    = {162},
  pages     = {21984--22014},
  year      = {2022},
  publisher = {PMLR},
  url       = {https://proceedings.mlr.press/v162/vargaftik22a.html}
}

@article{malinovskii2024pv,
  title={Pv-tuning: Beyond straight-through estimation for extreme llm compression},
  author={Malinovskii, Vladimir and Mazur, Denis and Ilin, Ivan and Kuznedelev, Denis and Burlachenko, Konstantin and Yi, Kai and Alistarh, Dan and Richtarik, Peter},
  journal={Advances in Neural Information Processing Systems},
  volume={37},
  pages={5074--5121},
  year={2024}
}

@article{wang2026lcqat,
  title   = {{LC-QAT}: Data-Efficient 2-Bit {QAT} for {LLM}s via
             Linear-Constrained Vector Quantization},
  author  = {Wang, Haoyu and Yu, Xingyu and Zhao, Haiyan and
             Wang, Fengxiang and Han, Xu},
  journal={ICML},
  year={2026}
}

@article{yi2026jacquant,
  title={JacQuant: STE-Free Quantization-Aware Training via Learned Jacobian Surrogates},
  author={Yi, Kai and Vivekraja, Vignesh and Khaitan, Harshit and Li, Steven},
  journal={arXiv preprint arXiv:2605.25469},
  year={2026}
}

@article{llama3,
  title={The llama 3 herd of models},
  author={Dubey, Abhimanyu and Jauhri, Abhinav and Pandey, Abhinav and Kadian, Abhishek and Al-Dahle, Ahmad and Letman, Aiesha and Mathur, Akhil and Schelten, Alan and Yang, Amy and Fan, Angela and others},
  journal={arXiv e-prints},
  pages={arXiv--2407},
  year={2024}
}

@article{yang2025qwen3,
  title={Qwen3 technical report},
  author={Yang, An and Li, Anfeng and Yang, Baosong and Zhang, Beichen and Hui, Binyuan and Zheng, Bo and Yu, Bowen and Gao, Chang and Huang, Chengen and Lv, Chenxu and others},
  journal={arXiv preprint arXiv:2505.09388},
  year={2025}
}

@article{diaconis1984asymptotics,
  title={Asymptotics of Graphical Projection Pursuit},
  author={Diaconis, Persi and Freedman, David},
  journal={The Annals of Statistics},
  volume={12},
  number={3},
  pages={793--815},
  year={1984},
  doi={10.1214/aos/1176346703}
}

@inproceedings{chen2025efficientqat,
  title={Efficientqat: Efficient quantization-aware training for large language models},
  author={Chen, Mengzhao and Shao, Wenqi and Xu, Peng and Wang, Jiahao and Gao, Peng and Zhang, Kaipeng and Luo, Ping},
  booktitle={Proceedings of the 63rd Annual Meeting of the Association for Computational Linguistics (Volume 1: Long Papers)},
  pages={10081--10100},
  year={2025}
}

@inproceedings{clark2018think,
  title={Think you have solved question answering? Try ARC, the AI2 reasoning challenge},
  author={Clark, Peter and Cowhey, Isaac and Etzioni, Oren and Khot, Tushar and Sabharwal, Ashish and Schoenick, Carissa and Tafjord, Oyvind},
  booktitle={arXiv preprint arXiv:1803.05457},
  year={2018}
}

@inproceedings{clark2019boolq,
  title={BoolQ: Exploring the surprising difficulty of natural yes/no questions},
  author={Clark, Christopher and Lee, Kenton and Chang, Ming-Wei and Kwiatkowski, Tom and Collins, Michael and Toutanova, Kristina},
  booktitle={Proceedings of the 2019 Conference of the North American Chapter of the Association for Computational Linguistics: Human Language Technologies, Volume 1 (Long and Short Papers)},
  pages={2924--2936},
  year={2019}
}

@inproceedings{bisk2020piqa,
  title={PIQA: Reasoning about physical commonsense in natural language},
  author={Bisk, Yonatan and Zellers, Rowan and Gao, Jianfeng and Choi, Yejin and others},
  booktitle={Proceedings of the AAAI Conference on Artificial Intelligence},
  volume={34},
  pages={7432--7439},
  year={2020}
}

@inproceedings{sap2019socialiqa,
  title={SocialIQA: Commonsense reasoning about social interactions},
  author={Sap, Maarten and Rashkin, Hannah and Chen, Derek and LeBras, Ronan and Choi, Yejin},
  booktitle={Proceedings of the 2019 Conference on Empirical Methods in Natural Language Processing and the 9th International Joint Conference on Natural Language Processing (EMNLP-IJCNLP)},
  pages={4463--4473},
  year={2019}
}

@inproceedings{zellers2019hellaswag,
  title={HellaSwag: Can a machine really finish your sentence?},
  author={Zellers, Rowan and Holtzman, Ari and Bisk, Yonatan and Farhadi, Ali and Choi, Yejin},
  booktitle={Proceedings of the 57th Annual Meeting of the Association for Computational Linguistics},
  pages={4791--4800},
  year={2019}
}

@inproceedings{merity2017pointer,
  title={Pointer Sentinel Mixture Models},
  author={Merity, Stephen and Xiong, Caiming and Bradbury, James and Socher, Richard},
  booktitle={International Conference on Learning Representations},
  year={2017}
}

@inproceedings{mihaylov2018can,
  title={Can a suit of armor conduct electricity? A new dataset for open book question answering},
  author={Mihaylov, Todor and Clark, Peter and Khot, Tushar and Sabharwal, Ashish},
  booktitle={Proceedings of the 2018 Conference on Empirical Methods in Natural Language Processing},
  pages={2381--2391},
  year={2018}
}

@inproceedings{sakaguchi2021winogrande,
  title={WinoGrande: An adversarial winograd schema challenge at scale},
  author={Sakaguchi, Keisuke and Bras, Ronan Le and Bhagavatula, Chandra and Choi, Yejin},
  booktitle={Proceedings of the AAAI Conference on Artificial Intelligence},
  volume={34},
  pages={8732--8740},
  year={2020}
}

@inproceedings{dettmers2022llmint8,
  title     = {{LLM.int8()}: 8-bit Matrix Multiplication for Transformers at Scale},
  author    = {Dettmers, Tim and Lewis, Mike and Belkada, Younes and Zettlemoyer, Luke},
  booktitle = {Advances in Neural Information Processing Systems},
  volume    = {35},
  year      = {2022},
  url       = {https://arxiv.org/abs/2208.07339}
}

@inproceedings{xiao2023smoothquant,
  title     = {{SmoothQuant}: Accurate and Efficient Post-Training Quantization for Large Language Models},
  author    = {Xiao, Guangxuan and Lin, Ji and Seznec, Mickael and Wu, Hao and Demouth, Julien and Han, Song},
  booktitle = {Proceedings of the 40th International Conference on Machine Learning},
  series    = {Proceedings of Machine Learning Research},
  volume    = {202},
  publisher = {PMLR},
  year      = {2023},
  url       = {https://proceedings.mlr.press/v202/xiao23c.html}
}

@inproceedings{frantar2022gptq,
  title     = {{GPTQ}: Accurate Post-Training Quantization for Generative Pre-trained Transformers},
  author    = {Frantar, Elias and Ashkboos, Saleh and Hoefler, Torsten and Alistarh, Dan},
  booktitle = {International Conference on Learning Representations},
  year      = {2023},
  url       = {https://arxiv.org/abs/2210.17323}
}

@inproceedings{lin2024awq,
  title     = {{AWQ}: Activation-aware Weight Quantization for On-Device {LLM} Compression and Acceleration},
  author    = {Lin, Ji and Tang, Jiaming and Tang, Haotian and Yang, Shang and Chen, Wei-Ming and Wang, Wei-Chen and Xiao, Guangxuan and Dang, Xingyu and Gan, Chuang and Han, Song},
  booktitle = {Proceedings of Machine Learning and Systems},
  year      = {2024},
  url       = {https://arxiv.org/abs/2306.00978}
}

@inproceedings{shao2024omniquant,
  title     = {{OmniQuant}: Omnidirectionally Calibrated Quantization for Large Language Models},
  author    = {Shao, Wenqi and Chen, Mengzhao and Zhang, Zhaoyang and Xu, Peng and Zhao, Lirui and Li, Zhiqian and Zhang, Kaipeng and Gao, Peng and Qiao, Yu and Luo, Ping},
  booktitle = {International Conference on Learning Representations},
  year      = {2024},
  url       = {https://openreview.net/forum?id=8Wuvhh0LYW}
}

@article{wang2023bitnet,
  title         = {{BitNet}: Scaling 1-bit Transformers for Large Language Models},
  author        = {Wang, Hongyu and Ma, Shuming and Dong, Li and Huang, Shaohan and Wang, Huaijie and Ma, Lingxiao and Yang, Fan and Wang, Ruiping and Wu, Yi and Wei, Furu},
  journal       = {arXiv preprint arXiv:2310.11453},
  year          = {2023},
  eprint        = {2310.11453},
  archivePrefix = {arXiv},
  url           = {https://arxiv.org/abs/2310.11453}
}

@article{ma2024era,
  title         = {The Era of 1-bit {LLM}s: All Large Language Models are in 1.58 Bits},
  author        = {Ma, Shuming and Wang, Hongyu and Ma, Lingxiao and Wang, Lei and Wang, Wenhui and Huang, Shaohan and Dong, Li and Wang, Ruiping and Xue, Jilong and Wei, Furu},
  journal       = {arXiv preprint arXiv:2402.17764},
  year          = {2024},
  eprint        = {2402.17764},
  archivePrefix = {arXiv},
  url           = {https://arxiv.org/abs/2402.17764}
}

@article{shang2023pbllm,
  title         = {{PB-LLM}: Partially Binarized Large Language Models},
  author        = {Shang, Yuzhang and Yuan, Zhihang and Wu, Qiang and Dong, Zhen},
  journal       = {arXiv preprint arXiv:2310.00034},
  year          = {2023},
  eprint        = {2310.00034},
  archivePrefix = {arXiv},
  url           = {https://arxiv.org/abs/2310.00034}
}

@inproceedings{huang2024billm,
  title     = {{BiLLM}: Pushing the Limit of Post-Training Quantization for {LLM}s},
  author    = {Huang, Wei and Liu, Yangdong and Qin, Haotong and Li, Ying and Zhang, Shiming and Liu, Xianglong and Magno, Michele and Qi, Xiaojuan},
  booktitle = {Proceedings of the 41st International Conference on Machine Learning},
  series    = {Proceedings of Machine Learning Research},
  volume    = {235},
  publisher = {PMLR},
  year      = {2024},
  url       = {https://arxiv.org/abs/2402.04291}
}

@inproceedings{xu2024onebit,
  title     = {{OneBit}: Towards Extremely Low-bit Large Language Models},
  author    = {Xu, Yuzhuang and Han, Xu and Yang, Zonghan and Wang, Shuo and Zhu, Qingfu and Liu, Zhiyuan and Liu, Weidong and Che, Wanxiang},
  booktitle = {Advances in Neural Information Processing Systems},
  year      = {2024},
  url       = {https://arxiv.org/abs/2402.11295}
}

@inproceedings{liu2024llmqat,
  title     = {{LLM-QAT}: Data-Free Quantization Aware Training for Large Language Models},
  author    = {Liu, Zechun and Oguz, Barlas and Zhao, Changsheng and Chang, Ernie and Stock, Pierre and Mehdad, Yashar and Shi, Yangyang and Krishnamoorthi, Raghuraman and Chandra, Vikas},
  booktitle = {Findings of the Association for Computational Linguistics: ACL 2024},
  pages     = {467--484},
  publisher = {Association for Computational Linguistics},
  year      = {2024},
  url       = {https://aclanthology.org/2024.findings-acl.26/}
}

@article{liu2025paretoq,
  title         = {{ParetoQ}: Scaling Laws in Extremely Low-bit {LLM} Quantization},
  author        = {Liu, Zechun and Zhao, Changsheng and Huang, Hanxian and Chen, Sijia and Zhang, Jing and Zhao, Jiawei and Roy, Scott and Jin, Lisa and Xiong, Yunyang and Shi, Yangyang and Xiao, Lin and Tian, Yuandong and Soran, Bilge and Krishnamoorthi, Raghuraman and Blankevoort, Tijmen and Chandra, Vikas},
  journal       = {arXiv preprint arXiv:2502.02631},
  year          = {2025},
  eprint        = {2502.02631},
  archivePrefix = {arXiv},
  url           = {https://arxiv.org/abs/2502.02631}
}

@article{lloyd1982least,
  title   = {Least Squares Quantization in {PCM}},
  author  = {Lloyd, Stuart P.},
  journal = {IEEE Transactions on Information Theory},
  volume  = {28},
  number  = {2},
  pages   = {129--137},
  year    = {1982},
  doi     = {10.1109/TIT.1982.1056489}
}

@article{max1960quantizing,
  title   = {Quantizing for Minimum Distortion},
  author  = {Max, Joel},
  journal = {IRE Transactions on Information Theory},
  volume  = {6},
  number  = {1},
  pages   = {7--12},
  year    = {1960},
  doi     = {10.1109/TIT.1960.1057548}
}

@inproceedings{dettmers2023qlora,
  title     = {{QLoRA}: Efficient Finetuning of Quantized {LLM}s},
  author    = {Dettmers, Tim and Pagnoni, Artidoro and Holtzman, Ari and Zettlemoyer, Luke},
  booktitle = {Advances in Neural Information Processing Systems},
  volume    = {36},
  year      = {2023},
  url       = {https://arxiv.org/abs/2305.14314}
}

@inproceedings{kim2024squeezellm,
  title     = {{SqueezeLLM}: Dense-and-Sparse Quantization},
  author    = {Kim, Sehoon and Hooper, Coleman and Gholami, Amir and Dong, Zhen and Li, Xiuyu and Shen, Sheng and Mahoney, Michael W. and Keutzer, Kurt},
  booktitle = {Proceedings of the 41st International Conference on Machine Learning},
  series    = {Proceedings of Machine Learning Research},
  volume    = {235},
  publisher = {PMLR},
  year      = {2024},
  url       = {https://arxiv.org/abs/2306.07629}
}

@inproceedings{egiazarian2024aqlm,
  title     = {Extreme Compression of Large Language Models via Additive Quantization},
  author    = {Egiazarian, Vage and Panferov, Andrei and Kuznedelev, Denis and Frantar, Elias and Babenko, Artem and Alistarh, Dan},
  booktitle = {Proceedings of the 41st International Conference on Machine Learning},
  series    = {Proceedings of Machine Learning Research},
  volume    = {235},
  publisher = {PMLR},
  year      = {2024},
  url       = {https://arxiv.org/abs/2401.06118}
}

@inproceedings{tseng2024quipsharp,
  title     = {{QuIP}$\#$: Even Better {LLM} Quantization with Hadamard Incoherence and Lattice Codebooks},
  author    = {Tseng, Albert and Chee, Jerry and Sun, Qingyao and Kuleshov, Volodymyr and De Sa, Christopher},
  booktitle = {Proceedings of the 41st International Conference on Machine Learning},
  series    = {Proceedings of Machine Learning Research},
  volume    = {235},
  pages     = {48630--48656},
  publisher = {PMLR},
  year      = {2024},
  url       = {https://proceedings.mlr.press/v235/tseng24a.html}
}

@inproceedings{chee2023quip,
  title     = {{QuIP}: 2-Bit Quantization of Large Language Models With Guarantees},
  author    = {Chee, Jerry and Cai, Yaohui and Kuleshov, Volodymyr and De Sa, Christopher},
  booktitle = {Advances in Neural Information Processing Systems},
  volume    = {36},
  year      = {2023},
  url       = {https://arxiv.org/abs/2307.13304}
}

@inproceedings{ashkboos2024quarot,
  title     = {{QuaRot}: Outlier-Free 4-Bit Inference in Rotated {LLM}s},
  author    = {Ashkboos, Saleh and Mohtashami, Amirkeivan and Croci, Maximilian L. and Li, Bo and Cameron, Pashmina and Jaggi, Martin and Alistarh, Dan and Hoefler, Torsten and Hensman, James},
  booktitle = {Advances in Neural Information Processing Systems},
  year      = {2024},
  url       = {https://arxiv.org/abs/2404.00456}
}

@inproceedings{liu2024spinquant,
  title     = {{SpinQuant}: {LLM} Quantization with Learned Rotations},
  author    = {Liu, Zechun and Zhao, Changsheng and Fedorov, Igor and Soran, Bilge and Choudhary, Dhruv and Krishnamoorthi, Raghuraman and Chandra, Vikas and Tian, Yuandong and Blankevoort, Tijmen},
  booktitle = {International Conference on Learning Representations},
  year      = {2025},
  url       = {https://arxiv.org/abs/2405.16406}
}

@article{zandieh2025turboquant,
  title         = {{TurboQuant}: Online Vector Quantization with Near-optimal Distortion Rate},
  author        = {Zandieh, Amir and Daliri, Majid and Hadian, Majid and Mirrokni, Vahab},
  journal       = {arXiv preprint arXiv:2504.19874},
  year          = {2025},
  eprint        = {2504.19874},
  archivePrefix = {arXiv},
  url           = {https://arxiv.org/abs/2504.19874}
}

@inproceedings{liu2024kivi,
  title     = {{KIVI}: A Tuning-Free Asymmetric 2bit Quantization for {KV} Cache},
  author    = {Liu, Zirui and Yuan, Jiayi and Jin, Hongye and Zhong, Shaochen and Xu, Zhaozhuo and Braverman, Vladimir and Chen, Beidi and Hu, Xia},
  booktitle = {Proceedings of the 41st International Conference on Machine Learning},
  series    = {Proceedings of Machine Learning Research},
  volume    = {235},
  pages     = {32332--32344},
  publisher = {PMLR},
  year      = {2024},
  url       = {https://proceedings.mlr.press/v235/liu24bz.html}
}

@inproceedings{hooper2024kvquant,
  title     = {{KVQuant}: Towards 10 Million Context Length {LLM} Inference with {KV} Cache Quantization},
  author    = {Hooper, Coleman and Kim, Sehoon and Mohammadzadeh, Hiva and Mahoney, Michael W. and Shao, Yakun Sophia and Keutzer, Kurt and Gholami, Amir},
  booktitle = {Advances in Neural Information Processing Systems},
  year      = {2024},
  url       = {https://arxiv.org/abs/2401.18079}
}

@article{austin2021program,
 title={Program Synthesis with Large Language Models},
 author={Austin, Jacob and Odena, Augustus and Nye, Maxwell and Bosma, Maarten and Michalewski, Henryk and Dohan, David and Jiang, Ellen and Cai, Carrie and Terry, Michael and Le, Quoc and Sutton, Charles},
 journal={arXiv preprint arXiv:2108.07732}, year={2021}, url={https://arxiv.org/abs/2108.07732}}

@article{chen2021evaluating,
 title={Evaluating Large Language Models Trained on Code},
 author={Chen, Mark and Tworek, Jerry and Jun, Heewoo and Yuan, Qiming and others},
 journal={arXiv preprint arXiv:2107.03374}, year={2021}, url={https://arxiv.org/abs/2107.03374}}

@article{cobbe2021training,
  title = {Training Verifiers to Solve Math Word Problems},
  author = {Cobbe, Karl and Kosaraju, Vineet and Bavarian, Mohammad and Chen, Mark and Jun, Heewoo and Kaiser, Lukasz and Plappert, Matthias and Tworek, Jerry and Hilton, Jacob and Nakano, Reiichiro and Hesse, Christopher and Schulman, John},
  journal = {arXiv preprint arXiv:2110.14168},
  year = {2021},
  url = {https://arxiv.org/abs/2110.14168}
}

\newpage
\begingroup
\setcounter{tocdepth}{2}
\small
\setlength{\parskip}{0pt}
\tableofcontents
\endgroup
\newpage
\beginappendix
\section{Experimental Details}
\label{app:experimental_details}
\subsection{Models, Training, and Evaluation}
\label{app:models_and_data}
\label{app:qat_hparams}
\label{app:evaluation_protocol}
We use pretrained \texttt{LLaMA-3.2-1B/3B}, \texttt{LLaMA3-8B}, and \texttt{Qwen3-1.7B} checkpoints with their native tokenizers \citep{llama3,yang2025qwen3}. The \texttt{1B} model supports the broad precision and component study; the larger models test transfer of the same frozen-table construction. Standard QAT starts from the corresponding full-precision checkpoint and samples a continuous \texttt{FineWeb-Edu} stream. The architecture, vocabulary, and positional encoding are unchanged. Training excludes the evaluation datasets.

\paragraph{Optimization.}
We train on eight \texttt{H100/A100-80G} GPUs, with eight sequences per GPU. The standard recipe uses \texttt{2048}-token sequences, effective batch size \texttt{64}, and at most \texttt{240K} updates, or approximately \texttt{31.5B} tokens. We use \algn{AdamW} with $\beta=(0.9,0.999)$, $\epsilon=10^{-8}$, zero weight decay, and global gradient-norm clipping at $1.0$. A cosine schedule follows \texttt{1000} linear-warmup steps and decays to zero. Learning rates are selected from $\{1,2,4\}\times10^{-5}$ on held-out validation data; enabled scale and gain parameters share the weights' learning rate. For the reported \texttt{8B} \algn{CanonQ-Omni} configuration, we reuse the best training hyperparameters selected at \texttt{3B} to reduce tuning costs. The training seed is \texttt{42}. The \texttt{8K/40K} continuation and {instruction-tuned coding and reasoning extension} use the separate protocols in \cref{app:long_context_8k,app:coding_protocol}.

\paragraph{Language modeling and downstream tasks.}
For \texttt{WikiText-2}, we concatenate tokenized test documents without added separators, split the stream into non-overlapping \texttt{2048}-token windows, and drop the trailing partial window. PPL is the exponential of mean next-token negative log-likelihood over the scored tokens. All methods use the same tokenized corpus and evaluation implementation.
Zero-shot accuracy is the unweighted mean over
\texttt{ARC-Easy} and \texttt{ARC-Challenge} \citep{clark2018think},
\texttt{BoolQ} \citep{clark2019boolq},
\texttt{PIQA} \citep{bisk2020piqa},
\texttt{SIQA} \citep{sap2019socialiqa},
\texttt{HellaSwag} \citep{zellers2019hellaswag},
\texttt{OpenBookQA} \citep{mihaylov2018can},
and \texttt{WinoGrande} \citep{sakaguchi2021winogrande},
using common prompts and answer normalization across methods.
Evaluation re-quantizes shadow weights and keeps the trained activation/cache operators enabled. Each result represents one training seed unless stated otherwise.

\subsection{Quantization Paths and Baselines}
\label{app:quantization_configurations}
\label{app:quantizer_formulations}
\Cref{tab:method_routing} summarizes which sources each method quantizes. \texttt{A16/KV16} denotes a floating-point path; other labels specify index bits. A canonical $(b,d)$ table has $2^{bd}$ vector codewords, so a $bd$-bit index represents $d$ scalars. Canonical A/K/V paths use $d=1$ throughout; method suffixes specify $d_W$.

\begin{table}[!htbp]
\caption{Quantization paths. ``Canonical'' means HDP, energy normalization, and a frozen Gaussian-reference table. All QAT paths are active during training and evaluation.}
\label{tab:method_routing}
\centering
\setlength{\tabcolsep}{4pt}
\begin{adjustbox}{max width=\linewidth}
\begin{tabular}{llll}
\toprule
Method & Weights & Activations & KV cache\\\midrule
\algn{ParetoQ} & learned {uniform} grid & dynamic {uniform} & dynamic {uniform}\\
\algn{VQLLM-QAT} & group-fitted raw codebook & dynamic {uniform} & dynamic {uniform}\\
\algn{CanonQ-W} & canonical & dynamic {uniform} & dynamic {uniform}\\
\algn{CanonQ-Omni} & canonical & canonical & canonical\\
\algn{HDP+LM PTQ} & canonical, no updates & canonical, no updates & {uniform} if enabled\\\bottomrule
\end{tabular}
\end{adjustbox}
\end{table}

\paragraph{Baseline scope.}
\algn{ParetoQ} supplies the weight quantizer, extended here with the {uniform} A/KV operators below \citep{liu2025paretoq}. \algn{VQLLM-QAT} fits raw-space weight tables by \algn{k-means/LBG}; matching its index rate and tuple dimension to \algn{CanonQ-W} changes both source coordinates and table construction. \algn{RTN} and \algn{GPTQ} specify weight PTQ procedures; their low-A/KV rows evaluate composite configurations. \algn{KIVI} specifies a cache codec with inherited W/A quantization \citep{frantar2022gptq,ashkboos2024quarot,liu2024kivi}. These composite rows measure end-to-end configurations and do not isolate the original weight or cache algorithm.

\paragraph{{Uniform weight and dynamic-source quantizers.}}
The \algn{ParetoQ} weight path uses a learned per-row scale $\alpha$: \texttt{W1} reconstructs $\alpha\operatorname{sign}(w)$, while \texttt{W2} rounds to $\frac{\alpha}{4}\{-3,-1,1,3\}$. For a dynamic A/KV group $x$, the implemented symmetric operator is
\begin{equation}
 M=\max_i|x_i|,\quad s=\frac{2^{b-1}-1}{M+\epsilon},\quad
 \widehat x=\frac{\operatorname{round}(sx)}{s+\epsilon}.
 \label{eq:sym_activation_quantizer}
\end{equation}
Its active cardinalities are \texttt{15/255/3} at \texttt{A4/A8/KV2}, compared with \texttt{16/256/4} for canonical scalar tables. Consequently, comparisons at the same index rate can still differ in cardinality, normalization, and backward behavior.

\paragraph{Canonical paths.}
All sources use the codec in \cref{eq:canonical_coordinates,eq:canonical_codec}, with source-specific $(g_S,b_S,d_S,R_S,\gamma_{S,\ell})$. Activations and KV states use energy normalization without an additional learned, absmax, or percentile clipping parameter. K and V are quantized independently and decoded to their original coordinates before attention. Assignment uses the identity STE, with the energy differentiated through encoding and decoding as in \cref{eq:ste_jacobian_summary}. A source with a fixed unit gain sets $\gamma_{S,\ell}=1$.

\subsection{HDP Maps and Frozen Tables}
\label{app:hdp_implementation}
\label{app:codebook_init}
\paragraph{Grouping and rotations.}
Canonical groups contain $g_W=g_A=128$ scalar coordinates and $g_K=g_V=64$. Each source uses one fixed map $R_S=P_SHD_S$ across layers. Seeds are $s_S=s_0+o_S$, with $s_0=0$ and offsets $o_W=0$, $o_A=1$, $o_K=o_V=2$; a seeded PRNG draws signs and then the permutation. Thus K and V share a map, while W/A/KV maps differ. The normalized fast Hadamard transform uses \texttt{hammer.ops.fast\_hadamard\_transform} with a \texttt{PyTorch} butterfly fallback. Power-of-two groups require no padding. Quantized linear layers include attention $q,k,v,o$ and MLP $w_1,w_2,w_3$ projections; token embeddings and the language-model head remain unrotated.

\paragraph{Offline construction and assignment.}
Scalar tables are fitted to $\Normal(0,1)$ using double-precision \algn{Lloyd--Max} iterations; vector tables use \algn{Lloyd/LBG} assignments and centroid updates. The fitting gap $\varepsilon_{\mathrm{vq}}$ in \cref{eq:canonical_fit_gap} explicitly accounts for numerical codebook optimization. The scalar \texttt{W1} levels are $\pm0.7978846$; the scalar \texttt{W2/KV2} levels are approximately $\pm0.4528$ and $\pm1.5104$. Exact A4/A8 levels and vector centers are serialized with the configuration. Tables remain frozen through QAT and evaluation.

For \texttt{W1}, the evaluated dimensions $d\in\{1,2,4,8\}$ give $K\in\{2,4,16,256\}$; for \texttt{W2}, $d\in\{1,2,4\}$ gives $K\in\{4,16,256\}$. Scalar \texttt{A4/A8/KV2} tables have \texttt{16/256/4} entries. Sources literally share a table only when $(b,d)$ and the codewords match. Scalar assignment uses sorted midpoint thresholds and \texttt{torch.bucketize} with \texttt{right=False}, so midpoint ties select the lower level. Vector assignment uses \texttt{argmin} and selects the lowest index on ties. The same rule is used during training and evaluation.

\subsection{Index Rate and Stored Footprint}
\label{app:bit_accounting}
The label \texttt{W2A4KV2} specifies an index-rate tuple of \texttt{2/4/2/2}; it excludes decoding factors, table storage, and unquantized modules. For source $S$ with $N_S$ scalar values, the stored rate and once-per-model allocation are
\begin{align}
 b_S^{\mathrm{stored}}&=b_S^{\mathrm{idx}}+
 \frac{B_S^{\mathrm{decode}}+B_S^{\mathrm{local\text{-}table}}+B_S^{\mathrm{aux}}}{N_S},
 \label{eq:stored_rate}\\
 B^{\mathrm{resident}}&=B^{\mathrm{global\text{-}table}}+B^{\mathrm{rotation}}+B^{\mathrm{constant}}.
 \label{eq:resident_storage}
\end{align}
Here auxiliary bits include offsets, padding, and alignment. Globally shared objects are counted once; a group-fitted codebook contributes to local-table storage.

\paragraph{Decoding metadata.}
A group of $g$ values with $r$ stored $q_s$-bit decoding scalars adds $rq_s/g$ bits per value. A static canonical weight group needs the folded factor $\alpha_g=\gamma_g\rho_g/\sqrt g$ in $\widehat u_g=\alpha_gR_g^\top\widehat z_g$. Storing this product requires one scalar; retaining norm and gain separately requires two. Table~\ref{tab:w2a4kv2_stored_bits} gives both cases for the declared scalar group sizes. The rates are representation accounting, not measured exported bytes.

\begin{table}[t]
\caption{Canonical \texttt{W2A4KV2} storage before padding and unquantized modules. Each decoding scalar is 16 bits; $r$ counts independent factors after folding. A rates apply only to materialized quantized activation buffers. Global tables are separate.}
\label{tab:w2a4kv2_stored_bits}
\centering
\setlength{\tabcolsep}{10pt}
\begin{tabular}{lrrrr}
\toprule
Source & Index bits & Group $g$ & Stored bits, $r=1$ & Stored bits, $r=2$\\\midrule
Weights & 2 & 128 & 2.125 & 2.250\\
Activations & 4 & 128 & 4.125 & 4.250\\
Keys / values & 2 & 64 & 2.250 & 2.500\\\bottomrule
\end{tabular}
\end{table}

The same accounting gives 2.25 bits per value for a symmetric two-bit cache with one 16-bit scale per 64 values. An asymmetric cache with a scale and offset uses $2+(q_s+q_z)/g$ bits per value; at $g=32$ and $q_s=q_z=16$, this is 3.0. Per-row \algn{ParetoQ} weight scales instead contribute $16\sum_\ell m_\ell/\sum_\ell m_\ell n_\ell$ bits per weight for matrices of shape $m_\ell\times n_\ell$. Total storage therefore accounts for both index bits and method-specific metadata.

\paragraph{Shared and group-fitted tables.}
A canonical $(b,d)$ table with $q_C$-bit components occupies $2^{bd}dq_C$ bits. At 16 bits per component, \texttt{W2-d1/d2/d4} tables require \texttt{8/64/2048} bytes and an \texttt{A4-d1} table requires \texttt{32} bytes. Each distinct exported table is counted once. In contrast, a raw-space table shared by each block of 128 output rows contributes
\begin{equation}
 B_W^{\mathrm{local\text{-}table}}=
 \sum_\ell\lceil m_\ell/128\rceil\,2^{bd}dq_C.
 \label{eq:direct_full_stored_rate}
\end{equation}
These table-sharing blocks differ from the 128-scalar canonical normalization groups. Any separately retained raw-codebook normalization factors add their own metadata; gains absorbed into raw codewords are not counted twice.

\paragraph{Deployment implications.}
HDP maps can be regenerated from seeds without storing dense matrices; exported signs, permutations, and replicated tables still occupy resident bytes. KV decoding factors persist with cached indices. Activation factors can remain on chip in a fused kernel, so their contribution depends on whether an activation buffer is materialized. The relevant deployment quantity is $M_{\mathrm{total}}(B,L)=M_W+M_{\mathrm{resident}}+M_{KV}(B,L)+M_A^{\mathrm{peak}}$, where $M_W$ includes unquantized modules. Actual memory and speed also depend on packing, lookup, and kernel boundaries.

\subsection{{Code Generation and Mathematical Reasoning}}
\label{app:coding_protocol}
\Cref{tab:coding} evaluates \texttt{MobileLLM-Pro-Instruct-1B} at \texttt{W2A16KV16}. Token embeddings and output/language-model heads remain full precision. \algn{CanonQ-d1}, \algn{VQLLM-QAT-d4}, and \algn{ParetoQ} use canonical scalar, raw vector, and uniform weight quantizers, respectively. All three arms start from the same instruction-tuned FP checkpoint and use it as a full-precision distillation teacher.

{\textbf{Checkpoint and learning-rate selection.} All quantized methods are evaluated at step \texttt{279236}. For each method with a learning-rate sweep, \texttt{GSM8K} exact match selects the setting used for all three benchmarks. The reported reasoning score therefore reflects \texttt{GSM8K}-guided selection, while both coding benchmarks evaluate the selected model without contributing to selection.}

\paragraph{{Coding metrics and decoding.}}
Evaluation uses a dialog prompt: \texttt{HumanEval} is zero-shot with maximum prompt/generation lengths \texttt{2000/1024}, while \texttt{MBPP} is three-shot with maximum generation length \texttt{256}. Greedy decoding generates one sample per task, with seed \texttt{42} and batch size \texttt{32}. The FP reference uses the same harness and decoding settings. Generated code runs in an isolated subprocess with a five-second timeout; a timeout or nonzero exit is a failure. Thus pass@1 is the fraction of solved tasks, and parentheses in \cref{tab:coding} give integer counts.

{\textbf{Mathematical-reasoning metric.} \texttt{GSM8K} exact match measures final-answer correctness on 1,319 problems. FP, \algn{ParetoQ}, \algn{VQLLM-QAT-d4}, and \algn{CanonQ-d1} solve 709, 453, 431, and 630 problems, respectively. We display scores to one decimal place and compute improvements from the counts before rounding. CanonQ gives net gains of 177 and 199 correct answers over the quantized baselines. Relative to \algn{ParetoQ}, this is $100(630-453)/1319=13.4$ accuracy points or $100(630/453-1)=39.1\%$ relative improvement. CanonQ retains $630/709=88.9\%$ of the FP score and recovers $(630-453)/(709-453)=69.1\%$ of the baseline-to-FP gap.}

{\textbf{Uncertainty and interpretation.} On \texttt{HumanEval}, CanonQ and FP solve 102 and 99 tasks, with 20 solved only by CanonQ and 17 only by FP. Exact McNemar testing gives $p=0.74$; the paired-bootstrap 95\% interval for the accuracy difference is $[-5.5,9.1]$ points. This interval spans both signs, and no equivalence margin was tested. On \texttt{MBPP}, the discordant counts are 17 and 88 ($p<10^{-12}$), consistent with the 14.2-point FP gap. These comparisons quantify variation across benchmark tasks under greedy decoding; training-seed variation is a separate quantity. For \texttt{GSM8K}, the available aggregate counts support accuracy and count comparisons; paired significance testing requires per-problem outcomes.}

\section{Supplementary Experimental Results}
\subsection{Marginal Alignment and Task Adaptation}
\label{app:canonical_marginal_validation}
We test the scalar-marginal consequence of canonicalization using a trained \texttt{LLaMA3-1B W2A4KV2} checkpoint. Weights come from the checkpoint; activations and KV states are captured at quantizer inputs on a fixed calibration batch. Each comparison uses identical groups, with $g_W=g_A=128$ and $g_K=g_V=64$, and identical energy normalization. Thus $z_{\mathrm{norot}}=\sqrt{g_S}x/(\|x\|_2+\epsilon)$ and $z_{\mathrm{HDP}}=R_Sz_{\mathrm{norot}}$ differ only by rotation. Measurements precede codeword assignment.

For each layer we compute scalar $W_1$ distance to $\Normal(0,1)$ and the KS distance $\sup_t|\widehat F(t)-\Phi(t)|$. Table~\ref{tab:hdp_dist} averages layers equally. HDP reduces activation $W_1$ by 90\%, key $W_1$ by 57\%, and value $W_1$ by 40\%; weights are already close and show a small increase. The range of source-wise means narrows from $0.015$--$0.101$ to $0.010$--$0.026$. This explains the shared marginal reference in \cref{fig:hdp_collapse}. {These statistics assess marginal alignment. Estimating $\Delta_{S,d}$ requires a $W_2$ assessment of the corresponding tuple law, including coordinate dependence when $d>1$.}

\begin{figure}[H]
\centering
\includegraphics[width=\linewidth]{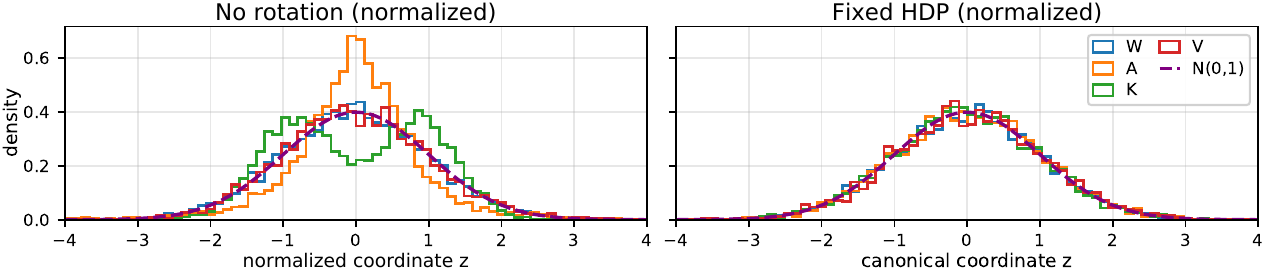}

\caption{Paired scalar marginals before and after HDP, using identical normalized groups. Dynamic A/K/V sources approach the Gaussian reference; weights remain near it. These scalar diagnostics illustrate alignment with the shared design reference. Full statistics appear in \cref{tab:hdp_dist}.}
\label{fig:hdp_collapse}
\end{figure}

\begin{table}[H]
\centering
\caption{Paired scalar-marginal distances to $\Normal(0,1)$ for a trained \texttt{LLaMA3-1B W2A4KV2} checkpoint. Values are layer means $\pm$ standard deviations, not uncertainty across training seeds. Both conditions use the same groups and normalization.}
\label{tab:hdp_dist}
\begin{tabular}{lcccc}
\toprule
&
\multicolumn{2}{c}{$W_1$ to $\Normal(0,1)$}
&
\multicolumn{2}{c}{KS to $\Normal(0,1)$}
\\
\cmidrule(lr){2-3}
\cmidrule(lr){4-5}
Source
& No rotation
& HDP
& No rotation
& HDP
\\
\midrule
Weights $W$
& $0.015\pm0.008$
& $0.017\pm0.012$
& $0.006\pm0.003$
& $0.010\pm0.006$
\\
Activations $A$
& $0.101\pm0.055$
& $\mathbf{0.010\pm0.004}$
& $0.039\pm0.019$
& $\mathbf{0.005\pm0.002}$
\\
Keys $K$
& $0.060\pm0.017$
& $\mathbf{0.026\pm0.013}$
& $0.035\pm0.011$
& $\mathbf{0.014\pm0.006}$
\\
Values $V$
& $0.030\pm0.014$
& $\mathbf{0.018\pm0.008}$
& $0.016\pm0.006$
& $\mathbf{0.009\pm0.003}$
\\
\bottomrule
\end{tabular}
\end{table}

\paragraph{Adaptation to a fixed codec.}
\label{app:ptq_vs_qat}
The zero-update control in \cref{tab:ptq_vs_qat} enables the canonical W/A codec directly on the pretrained \texttt{1B} model, with scalar tables, group size \texttt{128}, rotation seed \texttt{0}, and no gradient updates or layer reconstruction. At \texttt{W2A4}, QAT reduces PPL from 7043.43 to 13.84 and raises accuracy by 15.88 points without changing table capacity. The result demonstrates the importance of model adaptation to this fixed codec. The \texttt{KV2} pair additionally changes {uniform} KV to canonical KV, so its improvement reflects both adaptation and the cache path.

\begin{table}[!htbp]
\caption{Zero-update canonical PTQ and trained \algn{CanonQ-Omni-d1} on \texttt{LLaMA3-1B}. The \texttt{W2A4} pair shares the W/A codec; the \texttt{KV2} pair also changes the cache codec.}
\label{tab:ptq_vs_qat}
\centering
\setlength{\tabcolsep}{7pt}
\begin{tabular}{llrr}
\toprule
Format & Method / cache path & PPL $\downarrow$ & Avg. ZS $\uparrow$\\\midrule
\texttt{FP} & Full precision & 9.75 & 58.54\\ \midrule
\texttt{W2A4} & \algn{HDP+LM PTQ} & 7043.43 & 38.83\\
\texttt{W2A4} & \algn{CanonQ-Omni-d1} QAT & 13.84 & 54.71\\ \midrule
\texttt{W2A4KV2} & \algn{HDP+LM PTQ}, {uniform} KV & 8868.91 & 36.76\\
\texttt{W2A4KV2} & \algn{CanonQ-Omni-d1} QAT, canonical KV & 15.42 & 53.56\\\bottomrule
\end{tabular}
\end{table}

\subsection{{Canonical versus Uniform Cache Quantization}}\label{app:cache_path_ablation}
{\Cref{tab:cache_path_ablation} compares cache codecs while keeping W/A quantization canonical. The uniform-KV variant uses the dynamic uniform operator in \cref{app:quantizer_formulations}. Canonical KV improves both PPL and eight-task accuracy at \texttt{A8} and \texttt{A4}, extending the benefit beyond activation handling. This comparison evaluates the complete cache design, including its transform, active levels, normalization, and backward graph.}

\begin{table}[!htbp]
\centering\small
\caption{{Cache-path ablation for \texttt{LLaMA3-1B CanonQ-Omni-d1}. W/A paths remain canonical; the comparison changes the complete cache path.}}
\label{tab:cache_path_ablation}
\setlength{\tabcolsep}{8pt}
\begin{tabular}{llrr}
\toprule
Format & Cache path & PPL $\downarrow$ & Avg. ZS $\uparrow$\\\midrule
\texttt{W2A8KV2} & {uniform} & 22.82 & 47.72\\
 & canonical & \textbf{15.70} & \textbf{53.45}\\\midrule
\texttt{W2A4KV2} & {uniform} & 21.65 & 48.30\\
 & canonical & \textbf{15.42} & \textbf{53.56}\\\bottomrule
\end{tabular}
\end{table}

\subsection{Per-task Results across Precisions and Models}
\label{subsec:precision_scaling}
Tables~\ref{tab:complete_w2a16_w1a16}--\ref{tab:crossmodel_tasks} report the complete eight-task scores. \texttt{ARC-e/c}, \texttt{HellaS.}, \texttt{OBQA}, and \texttt{Wino.} abbreviate \texttt{ARC-Easy/Challenge}, \texttt{HellaSwag}, \texttt{OpenBookQA}, and \texttt{WinoGrande}. The mean weights the eight tasks equally.

\paragraph{Weight-only and activation precision.}
In \cref{tab:complete_w2a16_w1a16}, \texttt{1B W2A16} gives closely matched results for \algn{ParetoQ} and \algn{CanonQ-W-d2}: 12.66/55.53 versus 12.86/55.45 (PPL/accuracy). At \texttt{W1A16}, \algn{CanonQ-W-d1} has the lowest PPL, while $d_W=2$ has the highest accuracy. The matched-$d_W=4$ comparison improves raw-space \algn{VQLLM-QAT} from 22.36/47.90 to 17.47/51.17. In \cref{tab:complete_w2a8_w1a8_w2a4}, extending canonicalization to activations preserves low PPL when the {uniform} activation path degrades: at \texttt{W2A4}, \algn{CanonQ-Omni-d1} achieves 13.84 PPL against 35.16 for \algn{CanonQ-W-d1}.

\begin{table}[!htbp]
\centering
\caption{\texttt{LLaMA3-1B} weight-only quantization. All task scores are zero-shot accuracies (\%); PPL is measured on \texttt{WikiText-2}.}
\label{tab:complete_w2a16_w1a16}
\resizebox{\linewidth}{!}{
\begin{tabular}{lcccccccccl}
\toprule
Method & \texttt{ARC-e} &  \texttt{ARC-c} & \texttt{BoolQ} & \texttt{PIQA} & \texttt{SIQA} & \texttt{HellaS.} & \texttt{OBQA} & \texttt{Wino.} & Avg. ZS $\uparrow$ & PPL $\downarrow$ \\
\midrule
FP16 Base & 64.57 & 42.85 & 65.35 & 75.05 & 44.91 & 64.31 & 50.00 & 61.25 & 58.54 & 9.75\\
\midrule\midrule

\texttt{W2A16}\\ \midrule
\algn{RTN} (PTQ) & 27.19 & 24.69 & 54.03 & 50.00 & 35.89 & 25.43 & 23.05 & 50.94 & 36.40 & 1.2e6 \\
\algn{GPTQ} (PTQ) & 26.95 & 25.31 & 61.66 & 49.12 & 36.67 & 25.47 & 29.69 & 50.78 & 38.21 & 1.2e4 \\
\algn{QuaRot} (PTQ) & 28.12 & 24.69 & 62.35 & 50.34 & 36.23 & 26.72 & 28.12 & 49.45 & 38.25 & 3.2e3\\
\midrule
\paretoq & 64.60 & 39.69 & 61.05 & 72.75 & 42.79 & 56.74 & 47.85 & 58.75 & \cellcolor{bestblue}{55.53} & \cellcolor{bestblue}{12.66}\\ \midrule
\algn{VQLLM-QAT-d4} & 62.89 & 40.62 & 62.65 & 72.61 & 42.63 & 54.43 & 46.09 & 59.06 & 55.12 & 13.77 \\ \midrule
\algn{CanonQ-W-d1} & 64.38 & 38.05 & 62.71 & 72.36 & 43.21 & 56.19 & 46.68 & 57.27 & 55.11 & 12.94\\
\algn{CanonQ-W-d2} & 62.77 & 41.02 & 63.91 & 71.68 & 43.21 & 56.35 & 45.31 & 59.38 & \cellcolor{secondblue}55.45 & \cellcolor{secondblue}12.86\\
\algn{CanonQ-W-d4} & 64.34 & 41.09 & 61.51 & 72.41 & 43.46 & 56.46 & 47.66 & 56.64 & \cellcolor{secondblue}55.45 & 12.96\\
\midrule\midrule

\texttt{W1A16}\\ \midrule
\algn{RTN} (PTQ) & 25.0 & 22.5 & 37.6 & 49.5 & 32.9 & 25.0 & 27.1 & 49.6 & 33.7 & 4.2e8\\
\algn{GPTQ} (PTQ) & 28.59 & 22.73 & 62.32 & 50.29 & 36.62 & 25.29 & 28.71 & 49.69 & 38.03 & 7.9e4\\
\algn{QuaRot} (PTQ) & 28.16 & 26.02 & 62.41 & 48.58 & 37.16 & 24.80 & 28.71 & 50.86 & 38.34 & 1.3e4\\ \midrule
\paretoq & 59.73 & 35.62 & 61.54 & 68.65 & 42.29 & 46.48 & 43.75 & 54.38 & \cellcolor{secondblue}51.55 & 18.13\\ \midrule
\algn{VQLLM-QAT-d4} & 53.55 & 31.72 & 60.22 & 65.33 & 40.87 & 40.03 & 37.11 & 54.38 & 47.90 & 22.36\\
\algn{VQLLM-QAT-d8} & 57.77 & 36.80 & 59.83 & 66.65 & 41.60 & 42.93 & 40.62 & 52.58 & 49.85 & 20.10\\ \midrule
\algn{CanonQ-W-d1} & 58.95 & 35.55 & 61.66 & 68.65 & 42.77 & 47.20 & 42.97 & 53.59 & 51.42 & \cellcolor{bestblue}{16.93}\\
\algn{CanonQ-W-d2} & 60.66 & 36.41 & 61.96 & 68.60 & 41.70 & 46.76 & 42.19 & 56.17 & \cellcolor{bestblue}{51.81} & \cellcolor{secondblue}17.41\\
\algn{CanonQ-W-d4} & 58.79 & 36.25 & 61.99 & 68.16 & 41.75 & 46.98 & 40.82 & 54.61 & 51.17 & 17.47\\
\algn{CanonQ-W-d8} & 57.77 & 34.63 & 59.68 & 68.10 & 41.63 & 44.85 & 42.38 & 55.47 & 50.56 & 18.79\\
\bottomrule
\end{tabular}
}
\end{table}

\begin{table}[!htbp]
\centering
\caption{\texttt{LLaMA3-1B} with quantized weights and activations, retaining the floating-point cache. The three panels vary weight and activation precision.}
\label{tab:complete_w2a8_w1a8_w2a4}
\resizebox{\linewidth}{!}{\begin{tabular}{lrrrrrrrrrr}
\toprule
Method & \texttt{ARC-e} & \texttt{ARC-c} & \texttt{BoolQ} & \texttt{PIQA} & \texttt{SIQA} & \texttt{HellaS.} & \texttt{OBQA} & \texttt{Wino.} & Avg. ZS $\uparrow$ & PPL $\downarrow$\\\midrule
FP base & 64.57 & 42.85 & 65.35 & 75.05 & 44.91 & 64.31 & 50.00 & 61.25 & 58.54 & 9.75\\ \midrule
\midrule\multicolumn{11}{l}{\texttt{W2A8}}\\\midrule
\algn{RTN} (PTQ) & 26.91 & 25.16 & 51.47 & 50.29 & 36.43 & 25.72 & 23.05 & 50.08 & 36.14 & 1.2e6\\
\algn{GPTQ} (PTQ) & 26.68 & 25.39 & 62.08 & 48.39 & 37.11 & 25.56 & 29.69 & 50.39 & 38.16 & 1.3e4\\
\algn{QuaRot} (PTQ) & 28.71 & 24.38 & 62.44 & 50.15 & 36.08 & 26.74 & 28.91 & 49.92 & 38.42 & 3.2e3\\ \midrule
\algn{ParetoQ} & 58.99 & 35.73 & 61.56 & 68.16 & 41.43 & 45.72 & 42.58 & 54.53 & 51.09 & 18.22\\ \midrule
\algn{VQLLM-QAT-d4} & 61.36 & 37.42 & 61.20 & 70.10 & 42.69 & 49.91 & 43.75 & 55.39 & 52.73 & 16.99\\ \midrule
\algn{CanonQ-W-d1} & 63.56 & 40.71 & 61.41 & 71.01 & 43.50 & 54.06 & 45.51 & 58.36 & \cellcolor{secondblue}54.76 & 14.38\\
\algn{CanonQ-W-d2} & 64.15 & 39.95 & 59.92 & 71.07 & 43.04 & 55.67 & 43.16 & 53.59 & 53.82 & 13.49\\
\algn{CanonQ-W-d4} & 61.06 & 39.86 & 62.11 & 70.20 & 42.49 & 50.69 & 43.36 & 53.83 & 52.95 & 16.01\\ \midrule
\algn{CanonQ-Omni-d1} & 65.84 & 40.20 & 62.35 & 72.14 & 43.85 & 55.99 & 47.85 & 56.56 & \cellcolor{bestblue}55.60 & \cellcolor{bestblue}13.25\\
\algn{CanonQ-Omni-d2} & 61.99 & 38.01 & 64.17 & 72.04 & 43.04 & 56.63 & 43.75 & 57.42 & 54.63 & \cellcolor{secondblue}13.42\\ \midrule
\midrule\multicolumn{11}{l}{\texttt{W1A8}}\\\midrule
\algn{RTN} (PTQ) & 28.24 & 25.47 & 39.39 & 51.46 & 36.04 & 24.70 & 23.63 & 50.55 & 34.94 & 1.1e6\\
\algn{GPTQ} (PTQ) & 27.77 & 23.05 & 62.38 & 50.98 & 36.47 & 25.09 & 29.69 & 49.30 & 38.09 & 1.0e5\\
\algn{QuaRot} (PTQ) & 28.01 & 25.62 & 62.41 & 48.88 & 37.30 & 24.89 & 28.32 & 51.41 & 38.36 & 1.3e4\\ \midrule
\algn{ParetoQ} & 55.62 & 32.35 & 61.92 & 66.11 & 41.13 & 40.65 & 40.82 & 50.62 & 48.65 & 21.27\\ \midrule
\algn{VQLLM-QAT-d4} & 51.90 & 31.50 & 62.35 & 63.95 & 40.27 & 36.76 & 35.94 & 51.25 & 46.74 & 25.98\\
\algn{VQLLM-QAT-d8} & 51.94 & 32.26 & 60.77 & 63.85 & 41.08 & 39.13 & 38.87 & 51.88 & 47.47 & 25.37\\ \midrule
\algn{CanonQ-W-d1} & 55.83 & 33.78 & 60.16 & 65.14 & 41.48 & 41.07 & 36.72 & 52.73 & 48.37 & 23.54\\
\algn{CanonQ-W-d2} & 55.83 & 36.57 & 61.83 & 65.41 & 41.23 & 42.37 & 39.84 & 52.50 & 49.45 & 21.23\\
\algn{CanonQ-W-d4} & 54.73 & 32.77 & 61.68 & 66.22 & 40.47 & 41.30 & 41.99 & 50.86 & 48.75 & 22.10\\
\algn{CanonQ-W-d8} & 54.65 & 34.88 & 59.83 & 67.40 & 42.19 & 43.75 & 40.43 & 53.59 & 49.59 & 20.54\\ \midrule
\algn{CanonQ-Omni-d1} & 58.15 & 36.99 & 60.25 & 68.10 & 42.54 & 47.21 & 42.19 & 54.53 & \cellcolor{bestblue}51.25 & \cellcolor{secondblue}17.76\\
\algn{CanonQ-Omni-d2} & 57.39 & 34.54 & 57.98 & 68.37 & 42.09 & 46.71 & 39.84 & 52.73 & \cellcolor{secondblue}49.96 & \cellcolor{bestblue}17.47\\ \midrule
\midrule\multicolumn{11}{l}{\texttt{W2A4}}\\\midrule
\algn{RTN} (PTQ) & 26.60 & 26.02 & 54.24 & 50.34 & 36.82 & 25.30 & 23.83 & 51.09 & 36.78 & 1.1e6\\
\algn{GPTQ} (PTQ) & 27.54 & 24.06 & 51.20 & 52.15 & 37.06 & 24.79 & 25.98 & 52.11 & 36.86 & 6.9e4\\
\algn{QuaRot} (PTQ) & 26.45 & 25.31 & 60.61 & 51.27 & 37.35 & 25.81 & 27.34 & 50.00 & 38.02 & 4.6e3\\ \midrule
\algn{ParetoQ} & 47.76 & 27.87 & 58.68 & 61.96 & 40.57 & 31.66 & 32.23 & 52.27 & 44.12 & 32.57\\ \midrule
\algn{VQLLM-QAT-d4} & 45.99 & 27.62 & 60.13 & 61.31 & 41.03 & 30.71 & 31.05 & 49.77 & 43.45 & 42.65\\ \midrule
\algn{CanonQ-W-d1} & 45.90 & 27.96 & 59.38 & 60.34 & 40.98 & 31.39 & 34.38 & 52.81 & 44.14 & 35.16\\
\algn{CanonQ-W-d2} & 47.76 & 27.79 & 56.64 & 61.75 & 39.67 & 31.56 & 33.01 & 51.09 & 43.66 & 35.20\\
\algn{CanonQ-W-d4} & 47.30 & 29.48 & 60.53 & 60.78 & 40.12 & 32.01 & 34.77 & 51.25 & 44.53 & 35.53\\ \midrule
\algn{CanonQ-Omni-d1} & 63.89 & 39.10 & 60.19 & 71.39 & 43.40 & 55.28 & 46.09 & 58.36 & \cellcolor{bestblue}54.71 & \cellcolor{bestblue}13.84\\
\algn{CanonQ-Omni-d2} & 64.40 & 39.78 & 62.04 & 70.64 & 42.24 & 54.85 & 45.90 & 56.41 & \cellcolor{secondblue}54.53 & \cellcolor{secondblue}14.51\\
\bottomrule\end{tabular}}
\end{table}

\paragraph{Joint compression at \texttt{1B}.}
Table~\ref{tab:complete_w2a8kv2_w2a4kv2} shows that \algn{CanonQ-Omni-d1} retains similar quality at \texttt{A8} and \texttt{A4} with \texttt{KV2}: 15.70/53.45 and 15.42/53.56. The corresponding \algn{ParetoQ} results fall from 47.58/45.13 to 90.57/40.47. At \texttt{W2A4KV2}, \algn{CanonQ-Omni-d1} improves all eight tasks over \algn{ParetoQ}; scalar canonical tables suffice for this operating point.

\begin{table}[!htbp]
\centering
\caption{\texttt{LLaMA3-1B} joint W/A/KV quantization. \algn{CanonQ-Omni} uses {canonical KV quantization}; other QAT methods use {uniform} KV.}
\label{tab:complete_w2a8kv2_w2a4kv2}
\resizebox{\linewidth}{!}{
\begin{tabular}{lcccccccccl}
\toprule
Method & \texttt{ARC-e} &  \texttt{ARC-c} & \texttt{BoolQ} & \texttt{PIQA} & \texttt{SIQA} & \texttt{HellaS.} & \texttt{OBQA} & \texttt{Wino.} & Avg. ZS $\uparrow$ & PPL $\downarrow$ \\
\midrule
FP16 Base & 64.57 & 42.85 & 65.35 & 75.05 & 44.91 & 64.31 & 50.00 & 61.25 & 58.54 & 9.75\\
\midrule\midrule

\texttt{W2A8KV2}\\ \midrule
 \algn{RTN} (PTQ) & 26.88 & 22.50 & 56.16 & 51.07 & 37.26 & 25.30 & 27.54 & 51.80 & 37.31 & 5.8e5\\
 \algn{GPTQ} (PTQ) & 27.42 & 26.41 & 61.24 & 49.76 & 37.70 & 25.72 & 25.59 & 52.03 & 38.23 & 5.3e4\\
 \algn{QuaRot} (PTQ) & 25.86 & 24.06 & 61.90 & 51.95 & 36.43 & 24.81 & 25.00 & 49.14 & 37.39 & 3.4e4\\ \midrule
 \algn{KIVI} & 27.73 & 23.36 & 37.59 & 52.34 & 38.67 & 26.75 & 23.83 & 48.52 & 34.85 & 1.1e5\\ \midrule
 \paretoq & 47.76 & 29.81 & 61.89 & 62.98 & 40.68 & 31.72 & 34.96 & 51.25 & 45.13 & 47.58 \\ \midrule
 \algn{VQLLM-QAT-d4} & 46.33 & 29.81 & 61.98 & 61.75 & 39.67 & 34.56 & 34.96 & 53.75 & 45.35 & 34.35\\ \midrule
 \algn{CanonQ-W-d1} & 51.65 & 31.67 & 61.41 & 64.33 & 40.32 & 37.06 & 33.59 & 51.25 & 46.41 & 30.07 \\
 \algn{CanonQ-W-d2} & 53.34 & 33.61 & 57.19 & 65.79 & 40.07 & 38.06 & 39.65 & 52.81 & 47.56 & 30.19 \\
 \algn{CanonQ-W-d4} & 50.76 & 31.25 & 61.59 & 64.71 & 40.47 & 37.57 & 37.89 & 51.02 & 46.91 & 31.27 \\ \midrule
 \algn{CanonQ-Omni-d1} & 60.22 & 38.85 & 60.53 & 70.26 & 43.90 & 53.56 & 45.51 & 54.77 & \cellcolor{bestblue}53.45 & \cellcolor{bestblue}15.70 \\
 \algn{CanonQ-Omni-d2} & 60.73 & 34.97 & 62.65 & 70.15 & 43.25 & 50.52 & 44.14 & 54.92 & \cellcolor{secondblue}52.67 & \cellcolor{secondblue}16.43 \\
\midrule\midrule

\texttt{W2A4KV2}\\ \midrule
 \algn{RTN} (PTQ) & 28.83 & 25.55 & 54.33 & 50.59 & 36.08 & 25.11 & 23.24 & 50.00 & 36.72 & 5.5e5\\
 \algn{GPTQ} (PTQ) & 28.12 & 23.91 & 58.77 & 52.64 & 37.30 & 25.23 & 24.41 & 52.27 & 37.83 & 9.5e4\\
 \algn{QuaRot} (PTQ) & 27.58 & 25.16 & 61.75 & 51.17 & 36.28 & 24.83 & 25.98 & 52.11 & 38.11 & 4.1e4 \\
 \midrule
 \algn{KIVI} & 26.48 & 25.16 & 37.80 & 52.69 & 37.40 & 25.80 & 25.00 & 50.94 & 35.16 & 3.2e5\\ \midrule
 \paretoq & 40.41 & 25.17 & 54.34 & 56.57 & 39.21 & 28.13 & 30.08 & 49.84 & 40.47 & 90.57 \\ \midrule
 \algn{VQLLM-QAT-d4} & 41.72 & 25.93 & 61.95 & 58.41 & 39.62 & 28.84 & 33.40 & 48.67 & 42.32 & 76.51 \\ \midrule
 \algn{CanonQ-W-d1} & 41.47 & 28.29 & 58.10 & 59.38 & 39.47 & 28.78 & 32.42 & 52.03 & 42.49 & 64.40 \\
 \algn{CanonQ-W-d2} & 42.57 & 29.90 & 56.10 & 60.78 & 39.92 & 29.22 & 36.33 & 48.83 & 42.95 & 64.48 \\
 \algn{CanonQ-W-d4} & 42.02 & 25.25 & 57.68 & 58.14 & 39.31 & 28.28 & 32.62 & 52.42 & 41.97 & 76.19 \\ \midrule
 \algn{CanonQ-Omni-d1} & 61.32 & 40.79 & 62.14 & 70.20 & 42.99 & 52.23 & 43.95 & 54.84 & \cellcolor{bestblue}53.56 & \cellcolor{bestblue}15.42 \\
 \algn{CanonQ-Omni-d2} & 63.47 & 36.74 & 59.34 & 70.20 & 43.30 & 52.48 & 44.34 & 55.39 & \cellcolor{secondblue}53.16 & \cellcolor{secondblue}15.90 \\
\bottomrule
\end{tabular}
}
\end{table}

\paragraph{Weight-only and joint compression at \texttt{3B}.}
Table~\ref{tab:complete_3b} places all \texttt{3B} precisions under one FP reference. The \texttt{W2A16} difference is small: \algn{CanonQ-W-d2} reaches 9.19/63.02 versus 9.29/62.84 for \algn{ParetoQ}. Under \texttt{W2A4KV2}, the separation grows to 11.16/60.04 versus 92.64/39.39. Dimension changes the PPL--accuracy tradeoff: $d_W=1$ yields lower PPL (10.93), while $d_W=2$ raises accuracy from 58.37 to 60.04. Neither dimension dominates both metrics.

\begin{table}[!htbp]
\caption{\texttt{LLaMA3-3B}: weight-only, weight--activation, and joint W/A/KV quantization. The same FP reference applies to all three panels. Accuracy is in percent; lower PPL is better.}
\label{tab:complete_3b}
\resizebox{\linewidth}{!}{
\begin{tabular}{lrrrrrrrrrr}
\toprule
Method & \texttt{ARC-e} & \texttt{ARC-c} & \texttt{BoolQ} & \texttt{PIQA} & \texttt{SIQA} & \texttt{HellaS.} & \texttt{OBQA} & \texttt{Wino.} & Avg. ZS $\uparrow$ & PPL $\downarrow$\\\midrule
FP base & {72.99} & {49.84} & {75.12} & {78.12} & {48.29} & {74.05} & {53.52} & {68.91} & {65.11} & {7.73}\\\midrule\midrule
\texttt{W2A16}\\ \midrule
 \algn{RTN} (PTQ) & 27.11 & 25.94 & 47.30 & 50.34 & 36.91 & 25.45 & 25.78 & 49.30 & 36.02 & 3.6e5 \\
 \algn{GPTQ} (PTQ) & 27.34 & 23.83 & 61.54 & 51.12 & 35.35 & 26.10 & 29.49 & 48.83 & 37.95 & 7.3e3\\
 \algn{QuaRot} (PTQ) & 28.01 & 25.70 & 56.43 & 50.24 & 37.45 & 25.61 & 26.17 & 50.08 & 37.46 & 3.3e2\\ \midrule
\paretoq & {73.97} & 50.00 & 68.42 & 75.83 & {46.78} & 68.30 & {55.47} & 63.91 & \cellcolor{secondblue}62.84 & \cellcolor{secondblue}9.29 \\ \midrule
\algn{CanonQ-W-d1} & 73.31 & {50.08} & 69.02 & {75.99} & 46.63 & {68.54} & 51.76 & 64.22 & 62.44 & 9.37\\
\algn{CanonQ-W-d2} & 73.56 & 49.38 & {72.48} & 75.78 & 45.95 & 68.42 & 53.32 & {65.23} & \cellcolor{bestblue}{63.02} & \cellcolor{bestblue}{9.19} \\
\midrule\midrule
\multicolumn{11}{l}{\texttt{W2A4}}\\\midrule
\algn{RTN} (PTQ) & 27.23 & 24.84 & 46.30 & 50.54 & 37.60 & 25.34 & 24.80 & 49.30 & 35.74 & 4.2e5\\
\algn{GPTQ} (PTQ) & 28.05 & 25.08 & 61.57 & 51.51 & 37.84 & 24.99 & 27.54 & 48.36 & 38.12 & 7.9e4\\
\algn{QuaRot} (PTQ) & 27.93 & 25.47 & 55.05 & 51.71 & 37.84 & 26.29 & 24.80 & 48.75 & 37.23 & 6.7e2\\\midrule
\algn{ParetoQ} & 62.80 & 40.03 & 55.98 & 68.10 & 44.46 & 53.89 & 45.90 & 54.06 & 53.15 & 17.88\\ \midrule
\algn{VQLLM-QAT-d4} & 49.32 & 30.24 & 61.86 & 63.09 & 40.12 & 34.92 & 37.70 & 53.12 & 46.30 & 30.71\\ \midrule
\algn{CanonQ-W-d1} & 62.50 & 36.64 & 60.82 & 69.38 & 42.09 & 50.72 & 38.67 & 53.52 & 51.79 & 16.52 \\
\algn{CanonQ-W-d2} & 64.10 & 37.58 & 59.53 & 68.80 & 42.44 & 52.99 & 45.31 & 53.98 & 53.09 & 18.12\\ \midrule
\algn{CanonQ-Omni-d1} & 72.17 & 46.88 & 69.17 & 75.70 & 44.56 & 66.26 & 52.73 & 61.80 & \cellcolor{bestblue}61.16 & \cellcolor{secondblue}10.66 \\
\algn{CanonQ-Omni-d2} & 72.26 & 47.47 & 67.84 & 76.08 & 46.22 & 66.14 & 50.78 & 62.34 & \cellcolor{secondblue}61.14 & \cellcolor{bestblue}10.50\\\midrule\midrule
\multicolumn{11}{l}{\texttt{W2A4KV2}}\\\midrule
\algn{RTN} (PTQ) & 27.42 & 25.62 & 50.84 & 50.44 & 37.35 & 25.15 & 26.37 & 51.02 & 36.78 & 1.1e6\\
\algn{GPTQ} (PTQ) & 27.62 & 26.33 & 59.83 & 51.07 & 36.77 & 25.45 & 27.54 & 48.44 & 37.88 & 8.9e4\\
\algn{QuaRot} (PTQ) & 27.81 & 24.84 & 61.75 & 50.34 & 36.04 & 25.59 & 27.73 & 49.53 & 37.95 & 1.5e4\\\midrule
\algn{ParetoQ} & 37.29 & 23.82 & 49.91 & 55.23 & 39.26 & 26.18 & 32.23 & 51.17 & 39.39 & 92.64\\ \midrule
\algn{VQLLM-QAT-d4} & 45.14 & 27.20 & 44.54 & 60.51 & 41.08 & 30.78 & 34.77 & 52.58 & 42.07 & 52.82\\ \midrule
\algn{CanonQ-W-d1} & 39.02 & 26.52 & 44.39 & 55.50 & 39.21 & 26.94 & 28.52 & 52.11 & 39.03 & 76.28\\
\algn{CanonQ-W-d2} & 43.75 & 30.07 & 55.31 & 59.05 & 39.11 & 34.06 & 33.98 & 52.97 & 43.54 & 30.24 \\
\midrule
\algn{CanonQ-Omni-d1} & 68.29 & 44.09 & 67.05 & 74.46 & 44.20 & 63.18 & 45.70 & 60.00 & \cellcolor{secondblue}58.37 & \cellcolor{bestblue}10.93\\
\algn{CanonQ-Omni-d2} & 70.62 & 45.62 & 69.20 & 74.61 & 45.46 & 64.38 & 50.78 & 59.61 & \cellcolor{bestblue}60.04 & \cellcolor{secondblue}11.16\\\bottomrule
\end{tabular}}
\end{table}

{\textbf{Transfer across model family and scale.} In \cref{tab:crossmodel_tasks}, \algn{CanonQ-Omni-d1} improves every task over \algn{ParetoQ} on both \texttt{Qwen3-1.7B} and \texttt{LLaMA3-8B}. The largest gains occur on \texttt{HellaSwag}: 17.93 and 42.68 points, respectively. Improvements also span the remaining reasoning and commonsense tasks, so the aggregate gain reflects broad recovery with task-dependent magnitudes. The different PPL and accuracy gaps to FP reinforce the value of evaluating likelihood and task performance together.}

\begin{table}[!htbp]
\centering
\caption{Cross-family and scale transfer: \texttt{Qwen3-1.7B W2A4} and \texttt{LLaMA3-8B W2A4KV2}.}
\label{tab:crossmodel_tasks}
\resizebox{\linewidth}{!}{
\begin{tabular}{lrrrrrrrrrr}
\toprule
\multicolumn{11}{l}{\texttt{Qwen3-1.7B, W2A4}}\\ \midrule
Method & \texttt{ARC-e} & \texttt{ARC-c} & \texttt{BoolQ} & \texttt{PIQA} & \texttt{SIQA} & \texttt{HellaS.} & \texttt{OBQA} & \texttt{Wino.} & Avg. ZS $\uparrow$ & PPL $\downarrow$\\\midrule
FP base & 68.79 & 41.48 & 79.12 & 71.39 & 44.73 & 59.51 & 38.67 & 61.09 & 58.10 & 16.33\\ \midrule
\algn{ParetoQ} & 41.01 & 25.25 & 60.76 & 59.43 & 41.48 & 28.67 & 30.66 & 51.56 & 42.35 & 43.56\\ \midrule
\algn{CanonQ-Omni-d1} & 58.11 & 34.12 & 62.38 & 67.13 & 43.04 & 46.60 & 43.95 & 52.34 & \cellcolor{bestblue}50.96 & \cellcolor{bestblue}17.80\\\midrule\midrule
\multicolumn{11}{l}{\texttt{LLaMA3-8B, W2A4KV2}}\\ \midrule
FP base & 80.95 & 57.88 & 83.34 & 81.14 & 49.34 & 79.50 & 55.56 & 73.43 & 70.14 & 6.14\\ \midrule
\algn{ParetoQ} & 36.36 & 26.71 & 39.98 & 55.33 & 37.91 & 26.15 & 30.36 & 50.16 & 37.87 & 134.36\\ \midrule
\algn{CanonQ-Omni-d1} & 70.86 & 42.04 & 71.30 & 73.80 & 45.26 & 68.83 & 43.65 & 62.74 & \cellcolor{bestblue}59.81 & \cellcolor{bestblue}9.41\\\bottomrule
\end{tabular}}
\end{table}

\paragraph{Where the aggregate gain comes from.}
Figure~\ref{fig:quality_decomposition} connects the precision sweep to its task-level pattern. As compression extends from \texttt{W2A8} to \texttt{W2A4KV2}, \algn{CanonQ-Omni-d1} stays within 13.25--15.42 PPL, while the other QAT paths degrade. In \texttt{W2A4}, all eight tasks improve over \algn{ParetoQ}; gains range from 1.51 points on \texttt{BoolQ} to 23.62 on \texttt{HellaSwag}. Thus the mean improvement is distributed across tasks, although its size depends strongly on the benchmark.

\begin{figure}[!htbp]
\centering
\includegraphics[width=\linewidth]{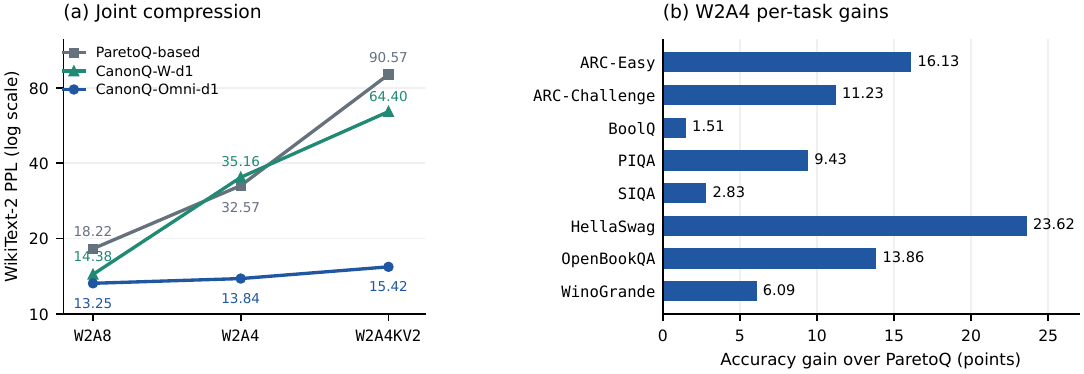}
\caption{{Quality under joint compression on \texttt{LLaMA3-1B}. (a) \algn{CanonQ-Omni-d1} retains low PPL as activation and cache precision decrease; the vertical axis is logarithmic. (b) Its \texttt{W2A4} improvement over \algn{ParetoQ} spans all eight tasks, with uneven gain magnitudes. Points summarize separately trained configurations; task scores appear in \cref{tab:complete_w2a8_w1a8_w2a4}.}}
\label{fig:quality_decomposition}
\end{figure}

\clearpage
\section{Context-Length Robustness after \texttt{8K} Continued QAT}
\label{app:long_context_8k}
\paragraph{Protocol.}
All context-length results use \texttt{LLaMA3-1B W2A4KV2}, continued QAT with \texttt{8192}-token sequences for \texttt{40K} updates, and evaluation at \texttt{2048/4096/8192} tokens. The FP reference is the base checkpoint evaluated at each length. Table~\ref{tab:long_context_8k} reports PPL and the same-length ratio $r_M(L)=\operatorname{PPL}_M(L)/\operatorname{PPL}_{\mathrm{FP}}(L)$.

\begin{table}[t]
\centering\small
\caption{\texttt{LLaMA3-1B W2A4KV2}: all quantized methods receive \texttt{8K} continued QAT for \texttt{40K} updates. \texttt{WikiText-2} PPL and its ratio to the same-length FP reference.}
\label{tab:long_context_8k}
\setlength{\tabcolsep}{9pt}
\begin{tabular}{lrrrrrr}
\toprule
& \multicolumn{3}{c}{PPL $\downarrow$} & \multicolumn{3}{c}{$\times$ FP $\downarrow$}\\
\cmidrule(lr){2-4}\cmidrule(lr){5-7}
Method & \texttt{2K} & \texttt{4K} & \texttt{8K} & \texttt{2K} & \texttt{4K} & \texttt{8K}\\\midrule
FP reference & 9.71 & 9.21 & 8.74 & 1.00 & 1.00 & 1.00\\ \midrule
\algn{ParetoQ} & 50.42 & 46.67 & 45.63 & 5.19 & 5.07 & 5.22\\
\algn{CanonQ-W} & 43.16 & 40.24 & 39.50 & 4.44 & 4.37 & 4.52\\
\algn{CanonQ-Omni} & \textbf{16.86} & \textbf{15.55} & \textbf{14.86} & \textbf{1.74} & \textbf{1.69} & \textbf{1.70}\\\bottomrule
\end{tabular}
\end{table}

\paragraph{Relative quality remains stable across the tested lengths.}
\algn{CanonQ-Omni} improves from 16.86 to 15.55 to 14.86 PPL as the evaluation window grows. Its reductions relative to \algn{ParetoQ} are $2.99\times$, $3.00\times$, and $3.07\times$; at \texttt{8K}, it also improves over \algn{CanonQ-W} by $2.66\times$. FP improves from 9.71 to 9.21 to 8.74, leaving \algn{CanonQ-Omni} at $1.69$--$1.74\times$ FP across lengths, versus $4.37$--$4.52\times$ for \algn{CanonQ-W} and $5.07$--$5.22\times$ for \algn{ParetoQ}. The ordering is consistent at all three lengths.

\paragraph{Interpretation.}
All evaluation windows lie within the \texttt{8K} continued-QAT training range. The experiment measures language-modeling quality across these windows. Quantized models receive continued QAT, while the FP checkpoint serves as a fixed same-length reference; the resulting ratios combine training and quantization effects. Longer windows provide additional preceding tokens and can change the scored-token set. We therefore assess context-length robustness through the consistent method ordering and approximately stable same-length FP ratio.

\section{Theory Details and Proofs}
\label{app:proofs}

Throughout this appendix, $Q_d=\Normal(0,I_d)$ and all tuple laws have finite
second moment.  Wasserstein distance is induced by the Euclidean norm.  For a
finite codebook $C\subset\mathbb R^d$, write
$d_C(v)=\dist(v,C)$ and let $\Pi_C(v)$ denote nearest-codeword assignment with
the deterministic tie rule used by the implementation.  The main transfer
result uses root mean-squared distortion because it is stable under $W_2$
without bounded support.

\subsection{The Haar Reference and the Implemented Hadamard Map}
\label{app:reference_and_hadamard}

\begin{proposition}[Joint Haar reference law]
\label{prop:haar_reference}
Let $R$ be Haar-random on $\mathsf O(g)$, let $u\ne0$ be fixed, and let
$1\le d<g$.  The first $d$ coordinates of
$V=\sqrt g\,Ru/\|u\|_2$ have density
\begin{equation}
    p_{g,d}(v)
    =
    \frac{\Gamma(g/2)}
    {\pi^{d/2}g^{d/2}\Gamma((g-d)/2)}
    \left(1-\frac{\|v\|_2^2}{g}\right)^{(g-d-2)/2},
    \qquad \|v\|_2<\sqrt g,
    \label{eq:joint_hypersphere_density}
\end{equation}
The law is independent of the direction of $u$, and for fixed $d$ it
converges to $Q_d$ in $W_2$ as $g\to\infty$.
\end{proposition}

\begin{proof}
Set $q=u/\|u\|_2$.  Haar invariance makes $Rq$ uniform on
$\mathbb S^{g-1}$.  If $G\sim\Normal(0,I_g)$, then
$Rq\overset{d}=G/\|G\|_2$.  For $X=(Rq)_{1:d}$,
\[
    \|X\|_2^2
    =\frac{\sum_{j\le d}G_j^2}{\sum_{j\le g}G_j^2}
    \sim\operatorname{Beta}\!\left(\frac d2,\frac{g-d}{2}\right).
\]
Rotational invariance in the first $d$ coordinates and normalization in polar
coordinates give
\[
    f_X(x)=
    \frac{\Gamma(g/2)}{\pi^{d/2}\Gamma((g-d)/2)}
    (1-\|x\|_2^2)^{(g-d-2)/2},
    \qquad \|x\|_2\le1.
\]
The change of variables $v=\sqrt g\,x$ yields
\cref{eq:joint_hypersphere_density}.  Finally,
\[
    V\overset{d}=
    \frac{G_{1:d}}{\sqrt{\|G\|_2^2/g}}
    \Rightarrow \Normal(0,I_d)
\]
by the law of large numbers and Slutsky's theorem.
Moreover, $\mathbb E\|V\|_2^2=d=\mathbb E_{Q_d}\|V\|_2^2$ for every $g$.
Weak convergence together with convergence of second moments is equivalent to
$W_2$ convergence, which proves the stated metric claim.
\end{proof}

The preceding proposition motivates the reference law but does not describe
the implemented transform.  The next result applies directly to one
random-sign Hadamard map and exposes the input-dependent quantities that must
be checked.

\begin{proposition}[One-HDP coordinate and covariance diagnostics]
\label{prop:hadamard_bridge}
Let $H\in\{\pm g^{-1/2}\}^{g\times g}$ be orthogonal, let
$D=\diag(\epsilon_1,\ldots,\epsilon_g)$ with independent Rademacher signs,
and let $\mathsf P$ be an independent uniform permutation.  For a fixed unit
vector $q$ independent of $(D,\mathsf P)$, define
$\widetilde Z=\sqrt g\,HDq$.  Every fixed Hadamard row satisfies
\begin{equation}
    \mathbb E_D \widetilde Z_i=0,
    \qquad
    \mathbb E_D \widetilde Z_i^2=1,
    \qquad
    \Pr_D(|\widetilde Z_i|\ge t)\le2e^{-t^2/2},
    \label{eq:hadamard_subgaussian}
\end{equation}
and, for a universal Berry--Esseen constant $C_{\mathrm{BE}}$,
\begin{equation}
    \sup_x|\Pr_D(\widetilde Z_i\le x)-\Phi(x)|
    \le
    C_{\mathrm{BE}}\sum_{j=1}^g|q_j|^3
    \le C_{\mathrm{BE}}\|q\|_\infty.
    \label{eq:hadamard_berry_esseen}
\end{equation}
If $I=(i_1,\ldots,i_d)$ is the ordered set of distinct Hadamard rows selected
by the first $d$ positions of $\mathsf P$, the deployed first tuple is
$Z_{1:d}:=\widetilde Z_I$, and
$\Sigma_I(q):=\mathbb E_D[\widetilde Z_I\widetilde Z_I^\top\mid I,q]$.  Then
\begin{equation}
    \mathbb E_{\mathsf P}
    \|\Sigma_I(q)-I_d\|_F^2
    =
    d(d-1)\frac{g\|q\|_4^4-1}{g-1}.
    \label{eq:hadamard_covariance_defect}
\end{equation}
\end{proposition}

\begin{proof}
Writing the unnormalized Hadamard signs as $s_{ij}\in\{\pm1\}$ gives $\widetilde Z_i=\sum_{j=1}^g s_{ij}\epsilon_jq_j$.
The mean and variance in \cref{eq:hadamard_subgaussian} follow from
independence and $\sum_jq_j^2=1$.  Moreover,
\[
    \mathbb E_D e^{\lambda \widetilde Z_i}
    =\prod_j\cosh(\lambda q_j)
    \le \exp\!\left(\frac{\lambda^2}{2}\sum_jq_j^2\right)
    =e^{\lambda^2/2}.
\]
The standard Chernoff argument gives the stated two-sided tail.  Applying the
Berry--Esseen inequality to the independent summands
$s_{ij}\epsilon_jq_j$ proves the first inequality in
\cref{eq:hadamard_berry_esseen}; the second uses
$\sum_j|q_j|^3\le\|q\|_\infty\sum_jq_j^2$.

For two distinct selected rows $i\ne k$,
\[
    [\Sigma_I(q)]_{ik}
    =\sum_{j=1}^g s_{ij}s_{kj}q_j^2,
    \qquad [\Sigma_I(q)]_{ii}=1.
\]
Fix $i$ and average the squared off-diagonal expression over $k\ne i$.
Hadamard-column orthogonality gives
\begin{align*}
    \sum_{k=1}^g
    \left(\sum_j s_{ij}s_{kj}q_j^2\right)^2
    &=
    \sum_{j,\ell}s_{ij}s_{i\ell}q_j^2q_\ell^2
    \sum_k s_{kj}s_{k\ell}
    =g\sum_jq_j^4.
\end{align*}
The $k=i$ term equals $(\sum_jq_j^2)^2=1$, so a uniformly sampled distinct
pair has expected squared covariance
$(g\|q\|_4^4-1)/(g-1)$.  There are $d(d-1)$ ordered off-diagonal entries,
which proves \cref{eq:hadamard_covariance_defect}.
\end{proof}

{\Cref{prop:hadamard_bridge} characterizes random-seed behavior for blocks independent of the sampled transform. For an adaptively trained checkpoint, \cref{thm:frozen_codebook_transfer} instead uses its actual post-HDP tuple law through $\Delta_d(P)$. The distinction is concrete: after fixing $R$, choosing $u=R^\top e_1$ gives $Ru=e_1$; sparse fixed blocks can also have non-Gaussian one-RHT laws. Related worst-case analyses establish sufficient multi-RHT constructions for scalar and vector guarantees \citep{benbasat2026rht}. CanonQ uses the random-seed bounds as structural diagnostics and the checkpoint-specific defect to quantify transfer of the frozen reference codebook.}

\begin{lemma}[Normalization regularization defect]
\label{lem:epsilon_defect}
{Sample a nonzero block $U\in\mathbb R^g$ by the chosen group law, then sample one of its $g/d$ tuples uniformly. Let $P_0$ and $P_\epsilon$ be the tuple laws under normalization by $\|U\|_2$ and $\|U\|_2+\epsilon$, respectively, using the same fixed orthogonal map. Then}
\begin{equation}
 {\frac{W_2(P_\epsilon,P_0)}{\sqrt d}\le
 \left[\mathbb E\left(\frac{\epsilon}{\|U\|_2+\epsilon}\right)^2\right]^{1/2},\qquad
 \Delta_d(P_\epsilon)\le\Delta_d(P_0)+
 \left[\mathbb E\left(\frac{\epsilon}{\|U\|_2+\epsilon}\right)^2\right]^{1/2}.}
 \label{eq:epsilon_defect}
\end{equation}
\end{lemma}
\begin{proof}
{Couple each regularized tuple to the corresponding unregularized tuple of the same block. Writing $a=\|U\|_2/(\|U\|_2+\epsilon)$ gives $Z_\epsilon=aZ_0$. Orthogonality implies $\|Z_0\|_2^2=g$ for the full block, so conditional mean squared tuple error is $d(1-a)^2$. This coupling bounds $W_2^2$; the triangle inequality gives the second claim.}
\end{proof}
{The bound isolates the regularizer from the HDP defect, but does not bound the latter. Zero blocks require the implementation's explicit convention; the differentiable-norm results below apply to nonzero blocks.}

\subsection{Root-Distortion Stability and Frozen-Table Transfer}
\label{app:transfer_proof}

\begin{lemma}[Root distortion is $W_2$-stable]
\label{lem:root_distortion_stability}
For any finite nonempty codebook $C$ and any tuple laws $P,Q$ with finite
second moment,
\begin{equation}
    |\mathcal R_{P,d}(C)-\mathcal R_{Q,d}(C)|
    \le\frac1{\sqrt d}W_2(P,Q).
    \label{eq:root_distortion_w2}
\end{equation}
Consequently,
$|\mathcal R_{P,b,d}^{\star}-\mathcal R_{Q,b,d}^{\star}|
\le W_2(P,Q)/\sqrt d$.
\end{lemma}

\begin{proof}
For any coupling $(X,Y)$ of $(P,Q)$, the distance to a fixed set is
one-Lipschitz.  The reverse triangle inequality in $L_2$ therefore gives
\begin{align*}
    \sqrt d\,
    |\mathcal R_{P,d}(C)-\mathcal R_{Q,d}(C)|
    &=
    \big|\|d_C(X)\|_{L_2}-\|d_C(Y)\|_{L_2}\big|\\
    &\le\|d_C(X)-d_C(Y)\|_{L_2}
    \le\|X-Y\|_{L_2}.
\end{align*}
Taking the infimum over couplings proves \cref{eq:root_distortion_w2}.  Applying
this inequality to every feasible $C$, taking infima, and exchanging $P,Q$
proves the optimal-distortion statement.
\end{proof}

\begin{proof}[Proof of \cref{thm:frozen_codebook_transfer}]
The product alphabet $A_b^d$ has $(2^b)^d=2^{bd}$ entries, so it is feasible
in the definition of $\mathcal R_{Q_d,b,d}^{\star}$.  This proves the first
part of \cref{eq:same_rate_root_dominance}.  From
\cref{eq:canonical_fit_gap} and
$\sqrt{a+e}\le\sqrt a+\sqrt e$,
\[
    \mathcal R_{Q_d,d}(\widehat C_{b,d})
    \le\mathcal R_{Q_d,b,d}^{\star}
    +\sqrt{\varepsilon_{\mathrm{vq}}},
\]
which proves its second part.

Let $\Delta=\Delta_d(P)$.  By
\cref{lem:root_distortion_stability},
\[
    \mathcal R_{P,d}(\widehat C_{b,d})
    \le\mathcal R_{Q_d,d}(\widehat C_{b,d})+\Delta,
    \qquad
    \mathcal R_{P,b,d}^{\star}
    \ge\mathcal R_{Q_d,b,d}^{\star}-\Delta.
\]
Subtracting and using the fitting gap yields
\cref{eq:root_regret_transfer}.  Applying the same two inequalities to
$\widehat C_{b,d}$ and $A_b^d$ gives
\begin{align*}
    \mathcal R_{P,d}(\widehat C_{b,d})-
    \mathcal R_{P,d}(A_b^d)
    &\le
    \mathcal R_{Q_d,d}(\widehat C_{b,d})-
    \mathcal R_{Q_d,d}(A_b^d)+2\Delta\\
    &=2\Delta-m_Q(A_b^d,\widehat C_{b,d}),
\end{align*}
which is \cref{eq:comparator_margin_transfer}.
\end{proof}

Root distortion gives a transfer bound under finite second moments. Squared distance is not globally Lipschitz on the Gaussian reference's unbounded support, so a bounded-support $W_1$ argument would not cover this codec.

\begin{lemma}[Correction for assign-then-gain decoding]
\label{lem:gain_correction}
Define
\[
    \mathcal R_{P,d}^{(\gamma)}(C)
    :=
    \left(\frac1d\mathbb E
    \|V-\gamma\Pi_C(V)\|_2^2\right)^{1/2},
    \qquad B_C:=\max_{c\in C}\|c\|_2.
\]
For the implementation that assigns with $C$ and scales only after assignment,
\begin{equation}
    |\mathcal R_{P,d}^{(\gamma)}(C)-
      \mathcal R_{P,d}^{(1)}(C)|
    \le\frac{|\gamma-1|B_C}{\sqrt d}.
    \label{eq:gain_correction}
\end{equation}
\end{lemma}

\begin{proof}
Apply the reverse triangle inequality in $L_2(P)$ to
$V-\gamma\Pi_C(V)$ and $V-\Pi_C(V)$.  Their difference is
$(1-\gamma)\Pi_C(V)$, whose norm is at most
$|1-\gamma|B_C$.  Division by $\sqrt d$ gives
\cref{eq:gain_correction}.
\end{proof}

Notice that $\gamma\Pi_C(V)$ is generally not the Euclidean projection onto
$\gamma C$: the assignment boundaries remain those of $C$.  Thus
same-rate dominance cannot simply be reused at non-unit gain.  An exact
projection onto the scaled table would instead assign by
$\gamma\Pi_C(V/\gamma)$ for $\gamma>0$.

\subsection{Gain-Aware Task-Loss Bounds}
\label{app:local_loss_proof}

{For a hard-quantized network, downstream code-cell switches can make the loss discontinuous as the intervention moves. Theorem~\ref{thm:local_loss_certificate} is valid on a smooth continuation or a neighborhood in which downstream assignments remain fixed and other operations are $C^3$. Its smoothness hypothesis must be checked; freezing the codebook does not imply it. For activation/KV interventions, stationarity in model parameters does not imply a small gradient in the intervention coordinates.}

\begin{proof}[Proof of \cref{thm:local_loss_certificate}]
Abbreviate $e=e_{C,\gamma}(u)$ and
$\eta=\eta_{C,\gamma}(z)$.  Taylor's theorem gives
\[
    \Loss(u+e)-\Loss(u)
    =\nabla\Loss(u)^\top e+
    \frac12e^\top\nabla^2\Loss(u)e+r_3,
    \qquad |r_3|\le\frac{M_3}{6}\|e\|_2^3.
\]
Orthogonality and \cref{eq:actual_block_perturbation} imply
$\|e\|_2=\rho(u)\|\eta\|_2/\sqrt g$.  Cauchy--Schwarz and the
operator-norm Hessian bound then give exactly
\cref{eq:full_local_loss_bound}.
\end{proof}

For a random block $U$, the magnitude of the quadratic term obeys
\[
    \mathbb E\left|\frac12e_{C,\gamma}(U)^\top
    \nabla^2\Loss(U)e_{C,\gamma}(U)\right|
    \le\frac12D^{\mathrm{curv}}_{S,\gamma}(C).
\]
This follows by applying the pointwise operator-norm inequality before taking
expectations.  It requires neither residual isotropy nor an almost-sure bound
on $\rho(U)^2\Lambda_U$.

\begin{corollary}[Simultaneous quantization at a common graph cut]
\label{cor:multisite_local_loss}
Let $e=(e_1,\ldots,e_m)$ concatenate intervention coordinates exposed at a
common cut of the computation graph, and regard $\Loss$ as a function of their
concatenation.  If
the full Hessian and third derivative satisfy the same operator-norm bounds as
in \cref{thm:local_loss_certificate}, then
\begin{equation}
    |\Loss(x+e)-\Loss(x)|
    \le
    \|\nabla\Loss(x)\|_2\|e\|_2
    +\frac{\Lambda_x}{2}\|e\|_2^2
    +\frac{M_3}{6}\|e\|_2^3.
    \label{eq:multisite_local_loss}
\end{equation}
\end{corollary}

\begin{proof}
Apply the same Taylor argument to the concatenated variable.  The quadratic
form $e^\top\nabla^2\Loss(x)e$ contains every cross-site block
$e_i^\top H_{ij}e_j$; \cref{eq:multisite_local_loss} bounds the full form and
does not discard these interactions.
For causally dependent internal sites, apply this result sequentially and
telescope the loss differences rather than treating the sites as coordinates
of one cut.
\end{proof}

\begin{lemma}[Coupled errors in linear and attention blocks]
\label{lem:coupled_linear_attention}

For a linear map with $\widehat W=W+E_W$ and $\widehat a=a+e_A$,
\[
\|\widehat W\widehat a-Wa\|_2
\le \|E_W\|_{\mathrm{op}}\|a\|_2
+\|W\|_{\mathrm{op}}\|e_A\|_2
+\|E_W\|_{\mathrm{op}}\|e_A\|_2.
\]
For attention with head dimension $h$, define
$S=QK^\top/\sqrt h$, $A=\operatorname{softmax}_{\mathrm{row}}(S)$,
and $Y=AV$. Let $\widehat Q=Q+E_Q$, $\widehat K=K+E_K$,
$\widehat V=V+E_V$, with corresponding $\widehat S$, $\widehat A$ and
$\widehat Y$. On the same attention mask,
\begin{align*}
\|\widehat S-S\|_F
&\le\frac{1}{\sqrt h}\bigl(
\|E_Q\|_F\|K\|_{\mathrm{op}}
+\|Q\|_{\mathrm{op}}\|E_K\|_F
+\|E_Q\|_F\|E_K\|_{\mathrm{op}}\bigr),\\
\|\widehat Y-Y\|_F
&\le \tfrac12\|\widehat S-S\|_F\|V\|_{\mathrm{op}}
+\|\widehat A\|_{\mathrm{op}}\|E_V\|_F.
\end{align*}
Masked logits are compared only on their common active entries. These are deterministic local bounds; no independence or isotropy of the errors is assumed.

\end{lemma}
\begin{proof}

Expanding $(W+E_W)(a+e_A)-Wa$ and applying the triangle inequality gives the first result. Similarly,
$\widehat S-S=(E_QK^\top+QE_K^\top+E_QE_K^\top)/\sqrt h$;
the Frobenius/operator-norm product inequalities give the logit bound.
For a probability vector $p$, the softmax Jacobian is
$J=\operatorname{diag}(p)-pp^\top$. Its absolute row sum is
$2p_i(1-p_i)\le1/2$; symmetry therefore implies
$\|J\|_{\mathrm{op}}\le1/2$. Integrating along the line segment between each pair of unmasked logit rows gives
$\|\widehat A-A\|_F\le\tfrac12\|\widehat S-S\|_F$.
Finally,
$\widehat Y-Y=(\widehat A-A)V+\widehat A E_V$,
which proves the output bound. A fixed causal mask is handled by applying the same argument on each row's active coordinates.

\end{proof}

\subsection{The STE Jacobian Induced by the Canonical Codec}
\label{app:ste_jacobian}

The following exact calculation is the method-specific part of the
optimization analysis.  It assumes the normalization scale is differentiated,
as in the stated fake-quantization graph, and that the nonnegative gain is an
independent parameter when differentiating with respect to $u$.

\begin{proposition}[Exact normalized-codebook STE Jacobian]
\label{prop:ste_jacobian}
Let $u\ne0$, set
$s(u)=\rho(u)/\sqrt g=(\|u\|_2+\epsilon)/\sqrt g$,
$z=Ru/s(u)$, $q=\Pi_C^{(d)}(z)$, and
$\widehat u=\gamma s(u)R^\top q$ for $\gamma\ge0$.  Under the identity STE
$Dq/Dz=I$, the surrogate Jacobian is
\begin{equation}
    J_{\mathrm{STE}}(u)
    =
    \gamma\left[I+R^\top(q-z)\nabla s(u)^\top\right].
    \label{eq:exact_ste_jacobian}
\end{equation}
With the assignment defect
$\delta_a(u)=\|q-z\|_2/\sqrt g$ and reconstruction defect
$\delta_r(u)=\|\gamma q-z\|_2/\sqrt g$,
\begin{equation}
    \|J_{\mathrm{STE}}-\gamma I\|_{\mathrm{op}}
    =\gamma\delta_a,
    \qquad
    \|J_{\mathrm{STE}}\|_{\mathrm{op}}
    \le\gamma(1+\delta_a),
    \qquad
    \|\widehat u-u\|_2=\rho(u)\delta_r.
    \label{eq:ste_jacobian_defects}
\end{equation}
The deviation from the gain-preserving reference output $\gamma u$ is
$\|\widehat u-\gamma u\|_2=\gamma\rho(u)\delta_a$.  If $\gamma$ is trained,
the gain-column defect is
\begin{equation}
    \left\|\partial_\gamma\widehat u-
    \partial_\gamma(\gamma u)\right\|_2
    =\|sR^\top(q-z)\|_2=\rho(u)\delta_a.
    \label{eq:gain_column_defect}
\end{equation}
\end{proposition}

\begin{proof}
Differentiate $z=Ru/s$ using the identity STE:
$dz=R\,du/s-z\,ds/s$. Since $sR^\top z=u$,
\begin{align*}
    d\widehat u
    &=\gamma\left(R^\top q\,ds+sR^\top dz\right)\\
    &=\gamma\left(du+R^\top(q-z)\,ds\right),
\end{align*}
which proves \cref{eq:exact_ste_jacobian}.  {For $u\ne0$, $\|\nabla s(u)\|_2=1/\sqrt g$. The matrix $A(u)=R^\top(q-z)\nabla s(u)^\top$ has rank at most one and satisfies $\|A(u)\|_{\mathrm{op}}=\|q-z\|_2/\sqrt g=\delta_a$. Hence $J_{\mathrm{STE}}-\gamma I=\gamma A(u)$ has operator norm $\gamma\delta_a$, and the Jacobian bounds follow.}

Finally,
\[
    \widehat u-u=sR^\top(\gamma q-z),
    \qquad
    \widehat u-\gamma u=\gamma sR^\top(q-z),
\]
and orthogonality gives the two forward-error identities.
Finally,
$\partial_\gamma\widehat u=sR^\top q$ and
$\partial_\gamma(\gamma u)=u=sR^\top z$, which gives
\cref{eq:gain_column_defect}.
\end{proof}

The two defects differ because the current implementation assigns in the
unscaled table $C$ and applies $\gamma$ afterward.  Assignment distortion
controls the STE Jacobian relative to its gain-scaled identity, whereas
reconstruction distortion controls the actual error relative to $u$.
If $z$ is partitioned into $g/d$ tuples $V_j$ and $P$ is obtained by sampling
one of them uniformly, then
\begin{equation}
    \delta_a(u)^2
    =\frac1g\sum_{j=1}^{g/d}
      \|\Pi_C(V_j)-V_j\|_2^2
    =D_{P,d}(C).
    \label{eq:assignment_defect_equals_distortion}
\end{equation}
Thus \cref{thm:frozen_codebook_transfer} directly bounds the RMS assignment
defect that appears in the STE Jacobian, subject only to the local sensitivity
factors in the next corollary.

\begin{corollary}[Codec-controlled STE bias and alignment]
\label{cor:codec_controlled_ste_bias}
Let $f$ be a downstream loss whose gradient is $L_f$-Lipschitz.  Define
$g_{\mathrm{STE}}=J_{\mathrm{STE}}^\top\nabla f(\widehat u)$.  Relative to
the gain-preserving smooth continuation $u\mapsto f(\gamma u)$,
\begin{align}
    \|g_{\mathrm{STE}}-\gamma\nabla f(\gamma u)\|_2
    &\le
    \gamma\delta_a\|\nabla f(\gamma u)\|_2
    +\gamma^2(1+\delta_a)L_f\rho(u)\delta_a
    =:B_C(u).
    \label{eq:codec_controlled_bias}
\end{align}
Consequently, with $g_{\mathrm{ref}}=\gamma\nabla f(\gamma u)$,
\begin{equation}
    \langle g_{\mathrm{STE}},g_{\mathrm{ref}}\rangle
    \ge\|g_{\mathrm{ref}}\|_2
    \bigl(\|g_{\mathrm{ref}}\|_2-B_C(u)\bigr).
    \label{eq:derived_ste_alignment}
\end{equation}
If $g_{\mathrm{ref}}\ne0$ and
$B_C(u)\le\kappa\|g_{\mathrm{ref}}\|_2$ for some $\kappa<1$, then
\begin{equation}
    \cos\angle(g_{\mathrm{STE}},g_{\mathrm{ref}})
    \ge\frac{1-\kappa}{1+\kappa}.
    \label{eq:ste_cosine_bound}
\end{equation}
\end{corollary}

\begin{proof}
Add and subtract
$J_{\mathrm{STE}}^\top\nabla f(\gamma u)$:
\begin{align*}
    g_{\mathrm{STE}}-\gamma\nabla f(\gamma u)
    &=
    (J_{\mathrm{STE}}^\top-\gamma I)\nabla f(\gamma u)
    +J_{\mathrm{STE}}^\top
    [\nabla f(\widehat u)-\nabla f(\gamma u)].
\end{align*}
Use \cref{eq:ste_jacobian_defects}, gradient Lipschitzness, and
$\|\widehat u-\gamma u\|_2=\gamma\rho(u)\delta_a$ to obtain
\cref{eq:codec_controlled_bias}.  Cauchy--Schwarz yields
\cref{eq:derived_ste_alignment}.  Under the relative-error condition,
$\|g_{\mathrm{STE}}\|_2\le(1+\kappa)\|g_{\mathrm{ref}}\|_2$ and the inner
product is at least
$(1-\kappa)\|g_{\mathrm{ref}}\|_2^2$, which proves
\cref{eq:ste_cosine_bound}.
\end{proof}

\begin{corollary}[RMS distortion controls the blockwise STE certificate]
\label{cor:rms_ste_bias}
Suppose on the analyzed region that
$0\le\gamma\le\bar\gamma$, $\rho(u)\le\bar\rho$,
$\|\nabla f(\gamma u)\|_2\le G$, and
$\delta_a(u)\le\bar\delta$.  With
$K:=\bar\gamma G+\bar\gamma^2(1+\bar\delta)L_f\bar\rho$, one has $B_C(u)\le K\delta_a(u)$.  If $P$ is the uniformly sampled tuple law
of the group, then
\begin{equation}
    B_C(u)^2\le K^2D_{P,d}(C).
    \label{eq:rms_ste_bias_from_distortion}
\end{equation}
For a random group $U$, taking expectation gives $\mathbb E B_C(U)^2\le K^2\mathbb E D_{P_U,d}(C)$; the tuple law must use the same group-sampling distribution.
\end{corollary}

\begin{proof}
Apply the regional bounds to \cref{eq:codec_controlled_bias}, then use
\cref{eq:assignment_defect_equals_distortion}.
\end{proof}

\begingroup
\paragraph{From source mismatch to STE bias.}
For the numerical reference table $C=\widehat C_{b,d}$, \cref{lem:root_distortion_stability,eq:canonical_fit_gap} give
\[
\mathcal R_{P,d}(C)\le \mathcal R_{Q_d,d}(C)+\Delta_d(P)
\le \mathcal R^\star_{Q_d,b,d}+\sqrt{\varepsilon_{\mathrm{vq}}}+\Delta_d(P).
\]
Under the regional bounds of \cref{cor:rms_ste_bias} and matching group/tuple sampling, it follows that
\begin{equation}
\mathbb E\!\left[B_C(U)^2\right]
\le K^2\!\left(\mathcal R^\star_{Q_d,b,d}+\sqrt{\varepsilon_{\mathrm{vq}}}+\Delta_d(P)\right)^2.
\label{eq:source_to_ste_bias}
\end{equation}
This makes the roles of reference distortion, numerical fitting, source mismatch, and downstream sensitivity explicit in one blockwise bias certificate.
\endgroup

{Comparing with the gain-scaled downstream gradient evaluated at the unquantized input $u$, the same decomposition and \cref{eq:ste_jacobian_defects} give}

\begin{equation}
    \|g_{\mathrm{STE}}-\gamma\nabla f(u)\|_2
    \le
    \gamma\delta_a\|\nabla f(u)\|_2
    +\gamma(1+\delta_a)L_f\rho(u)\delta_r.
    \label{eq:two_defect_bias}
\end{equation}
Assignment error $\delta_a$ controls the Jacobian defect, while reconstruction error $\delta_r$ controls forward error relative to $u$. Optimizing only the latter can leave the former large when gain is applied after assignment.

\begin{proposition}[Cellwise shadow objective]
\label{prop:cellwise_shadow}
Let $f:\mathbb R^d\to\mathbb R$ have $L_f$-Lipschitz gradient and let $V_k$
be the open Voronoi cell of codeword $c_k$.  Define
\[
    \mathcal S_Cf(z)
    :=f(c_k)+\langle\nabla f(c_k),z-c_k\rangle,
    \qquad z\in V_k.
\]
Inside each cell,
$\nabla\mathcal S_Cf(z)=\nabla f(\Pi_C(z))$, exactly the identity-STE
direction for a fixed-scale quantizer, and
\begin{equation}
    |\mathcal S_Cf(z)-f(z)|
    \le\frac{L_f}{2}\|z-\Pi_C(z)\|_2^2.
    \label{eq:cellwise_shadow_error}
\end{equation}
\end{proposition}

\begin{proof}
The gradient statement follows because $\mathcal S_Cf$ is affine on $V_k$.
The smoothness remainder for the first-order expansion of $f$ at $c_k$ gives
\cref{eq:cellwise_shadow_error}.
\end{proof}

The cellwise construction relates local STE modeling error to codebook distortion. It can be discontinuous across cell boundaries, so its guarantee is local to each cell.

For a full network, let $\mathcal B_\xi(\omega,e)$ denote the backward map when
the concatenated perturbation vector $e$ includes forward residuals,
local STE-Jacobian deviations, and trainable-quantizer Jacobian-column errors, with
$\mathcal B_\xi(\omega,0)=\nabla f_\xi(\omega)$.  The following elementary
lemma explains how the sitewise certificates aggregate.

\begin{lemma}[From aggregate codec errors to global STE bias]
\label{lem:ste_bias_from_codec}
{Assume $\mathbb E_\xi\nabla f_\xi(\omega)=\nabla F(\omega)$, $\mathbb E_\xi\|e(\xi)\|_2^2<\infty$, and the required gradients are integrable. If, on the analyzed region,}
\begin{equation}
    \|\mathcal B_\xi(\omega,e)-
      \mathcal B_\xi(\omega,0)\|_2
    \le\kappa\|e\|_2,
    \label{eq:backprop_stability}
\end{equation}
then for
$G(\omega;\xi)=\mathcal B_\xi(\omega,e(\xi))$ and
$F(\omega)=\mathbb E_\xi f_\xi(\omega)$,
\begin{equation}
    \|\mathbb E_\xi G(\omega;\xi)-\nabla F(\omega)\|_2^2
    \le\kappa^2\mathbb E_\xi\|e(\xi)\|_2^2.
    \label{eq:global_codec_bias}
\end{equation}
The exact blockwise bound in \cref{eq:codec_controlled_bias} supplies a
sharper local value when the normalization graph is modeled explicitly.
\end{lemma}

\begin{proof}
Jensen's inequality followed by \cref{eq:backprop_stability} gives
\begin{align*}
    \|\mathbb E_\xi[G-\nabla f_\xi]\|_2^2
    \le\mathbb E_\xi\|G-\nabla f_\xi\|_2^2
    \le\kappa^2\mathbb E_\xi\|e(\xi)\|_2^2.
\end{align*}
\end{proof}

\subsection{Distortion-Controlled Optimization Consequences}
\label{app:ste_convergence}

{For the SGD update in \cref{thm:distortion_controlled_ste}, the following analysis separates optimization and sampling error from the codec-dependent bias certificate. The result concerns the smooth reference objective under the stated update rule; the experimental optimizer is AdamW.}

\begin{proof}[Proof of \cref{thm:distortion_controlled_ste}]
Let $g_t=\nabla F(\omega_t)$,
$\bar G_t=\mathbb E_tG_t$, and
$b_t=\bar G_t-g_t$.  $L$-smoothness gives
\begin{equation}
    \mathbb E_tF(\omega_{t+1})
    \le
    F(\omega_t)-\alpha_t\langle g_t,g_t+b_t\rangle
    +\frac{L\alpha_t^2}{2}\mathbb E_t\|G_t\|_2^2.
    \label{eq:biased_ste_descent_start}
\end{equation}
The variance condition implies
\[
    \mathbb E_t\|G_t\|_2^2
    \le2\|g_t\|_2^2+2\|b_t\|_2^2+\sigma^2.
\]
Also,
$-\langle g_t,b_t\rangle
\le\tfrac14\|g_t\|_2^2+\|b_t\|_2^2$.
Substituting both inequalities into
\cref{eq:biased_ste_descent_start} and using
$L\alpha_t\le1/4$ yields
\begin{equation}
    \mathbb E_tF(\omega_{t+1})
    \le
    F(\omega_t)-\frac{\alpha_t}{2}\|g_t\|_2^2
    +\frac{5\alpha_t}{4}\|b_t\|_2^2
    +\frac{L\alpha_t^2\sigma^2}{2}.
    \label{eq:biased_ste_one_step}
\end{equation}
Take total expectation, sum from $t=0$ to $T-1$, use
$F(\omega_T)\ge F_\star$, and substitute the codec-bias bound.  Multiplying
the resulting inequality by two and dividing by $S_T$ gives
\cref{eq:distortion_controlled_stationarity}.
\end{proof}

{For the constant schedule $\alpha_t=c/\sqrt T\le1/(4L)$, $S_T=c\sqrt T$ and $\sum_{t<T}\alpha_t^2=c^2$. The optimization and sampling terms in \cref{eq:distortion_controlled_stationarity} therefore decay as $T^{-1/2}$, leaving the weighted codec-bias certificate. If that certificate also vanishes, the bound establishes stationarity of the smooth reference; otherwise it gives a residual upper bound determined by the codec and sensitivity conditions.}

\begin{corollary}[{Local PL bound with codec-dependent residual}]
\label{cor:pl_distortion_floor}
Suppose, in addition, that the iterates remain in a region satisfying the PL
inequality
\begin{equation}
    \frac12\|\nabla F(\omega)\|_2^2
    \ge\mu(F(\omega)-F_\star),
    \label{eq:local_pl}
\end{equation}
and that $\|b_t\|_2^2\le\overline B^2$.  For a constant
$0<\alpha\le\min\{1/(4L),1/\mu\}$,
\begin{equation}
    \mathbb E[F(\omega_T)-F_\star]
    \le
    (1-\mu\alpha)^T(F(\omega_0)-F_\star)
    +\frac{5\overline B^2}{4\mu}
    +\frac{L\alpha\sigma^2}{2\mu}.
    \label{eq:pl_codec_floor}
\end{equation}
\end{corollary}

\begin{proof}
Apply \cref{eq:local_pl} to \cref{eq:biased_ste_one_step} and take total
expectations:
\[
    \Delta_{t+1}
    \le(1-\mu\alpha)\Delta_t
    +\frac{5\alpha}{4}\overline B^2
    +\frac{L\alpha^2\sigma^2}{2},
    \qquad
    \Delta_t:=\mathbb E[F(\omega_t)-F_\star].
\]
Unrolling the geometric recurrence and upper-bounding the finite geometric
sum by $(\mu\alpha)^{-1}$ proves \cref{eq:pl_codec_floor}.
\end{proof}

\subsection{Voronoi Margins, Code Switching, and Index Stability}
\label{app:assignment_stability}

\begin{lemma}[Assignment-switch margin]
\label{lem:assignment_switch}
For a point $v$ with unique nearest codeword $c_{a(v)}$, define
\begin{equation}
    m_C(v)
    :=
    \min_{j\ne a(v)}
    \frac{\|v-c_j\|_2^2-\|v-c_{a(v)}\|_2^2}
    {2\|c_j-c_{a(v)}\|_2}.
    \label{eq:voronoi_margin}
\end{equation}
Set $m_C(v)=0$ when the nearest-codeword assignment is tied.
Then $m_C(v)$ is the distance from $v$ to the boundary of its Voronoi cell.
If $\|v'-v\|_2<m_C(v)$, then $a(v')=a(v)$.  Consequently, if
$\|V'-V\|_2\le r$ almost surely,
\begin{equation}
    \Pr(a(V')\ne a(V))
    \le\Pr(m_C(V)\le r).
    \label{eq:switch_probability}
\end{equation}
\end{lemma}

\begin{proof}
For $j\ne a(v)$, the bisector between $c_{a(v)}$ and $c_j$ is a hyperplane.
The signed Euclidean distance from $v$ to that hyperplane is precisely the
corresponding ratio in \cref{eq:voronoi_margin}.  The minimum is therefore the
distance to the boundary of the current cell.  The open ball of that radius
lies inside the cell, proving both claims.
\end{proof}

For the idealized $\epsilon=0$ normalization
$z(u)=\sqrt g Ru/\|u\|_2$, if $\|h\|_2\le\|u\|_2/2$, then
\[
    \|z(u+h)-z(u)\|_2
    \le\frac{2\sqrt g\|h\|_2}{\|u\|_2}.
\]
Let $z_j(u)\in\mathbb R^d$ denote tuple $j$.  A sufficient condition that one
shadow update preserve every tuple index is
\[
    \frac{2\sqrt g\|h\|_2}{\|u\|_2}
    <\min_j m_C(z_j(u)).
\]
{This inequality certifies unchanged assignments when it holds. Violating it permits, but does not imply, switching or chattering.}

\begin{corollary}[Stability of the code-index distribution]
\label{cor:index_distribution_stability}
Let $a_C$ be the nearest-codeword index map, let $P,Q$ be tuple laws, and let
$V\sim Q$.  Then
\begin{equation}
    \operatorname{TV}((a_C)_\#P,(a_C)_\#Q)
    \le
    \inf_{r>0}
    \left\{
        \frac{W_2(P,Q)^2}{r^2}
        +Q(m_C(V)\le r)
    \right\}.
    \label{eq:index_distribution_stability}
\end{equation}
\end{corollary}

\begin{proof}
Take a coupling $(X,Y)$ of $(P,Q)$.  If $\|X-Y\|_2\le r$ and
$m_C(Y)>r$, \cref{lem:assignment_switch} gives $a_C(X)=a_C(Y)$.  Hence
\[
    \Pr(a_C(X)\ne a_C(Y))
    \le\Pr(\|X-Y\|_2>r)+Q(m_C(Y)\le r).
\]
Markov's inequality bounds the first term by
$\mathbb E\|X-Y\|_2^2/r^2$.  Total variation between pushforwards is at most
the mismatch probability under any coupling.  Infimizing first over couplings
and then over $r$ proves \cref{eq:index_distribution_stability}.
\end{proof}

{Code-index stability is governed jointly by transport distance and the probability mass near Voronoi boundaries. The bound separates these contributions: $W_2$ controls displacement between tuple laws, while the boundary term controls the sensitivity of discrete assignments to that displacement.}

\end{document}